\documentclass[letterpaper]{article}
\usepackage{aaai24} 
\usepackage{times} 
\usepackage{helvet} 
\usepackage{courier}
\usepackage[hyphens]{url}
\usepackage{graphicx}
\usepackage{natbib} 
\usepackage{caption}
\usepackage{algorithm}
\usepackage{algorithmic}
\usepackage{newfloat}
\usepackage{listings}

\DeclareCaptionStyle{ruled}{labelfont=normalfont,labelsep=colon,strut=off} 
\floatstyle{ruled}
\newfloat{listing}{tb}{lst}{}
\floatname{listing}{Listing}

\usepackage{amssymb}
\usepackage{amsmath}
\usepackage{tabularray}
\usepackage{mathtools}
\usepackage{xcolor}
\usepackage{longtable}
\usepackage{supertabular,booktabs}
\usepackage{fontawesome} 
\usepackage{array}
\usepackage{multirow}

\title{Causal Adversarial Perturbations for Individual Fairness and Robustness in Heterogeneous Data Spaces}
\author{
    Ahmad-Reza Ehyaei \textsuperscript{\rm 1},
    Kiarash Mohammadi \textsuperscript{\rm 2}\textsuperscript{\rm 3},
    Amir-Hossein Karimi \textsuperscript{\rm 1}, 
    Samira Samadi \textsuperscript{\rm 1},
    Golnoosh Farnadi \textsuperscript{\rm 2}\textsuperscript{\rm 3}\textsuperscript{\rm 4}
}
\affiliations {
    \textsuperscript{\rm 1} Max Planck Institute for Intelligent Systems Tübingen
Tübingen, Germany\\
    \textsuperscript{\rm 2} Université de Montréal, Montréal, Canada\\
    \textsuperscript{\rm 3} Mila - Québec AI Institute, Montréal, Canada 
     \textsuperscript{\rm 4} McGill University \\
    
    ahmadreza.ehyaei@uni-tuebingen.de, 
    kiarash.mohammadi@mila.quebec,
    amir@tue.mpg.de,
    ssamadi@tuebingen.mpg.de, 
    farnadig@mila.quebec
}

\begin{document}
\newcommand{\variable}[1]{\textsc{#1}}
\newcommand{\ahmad}[1]{{\color{blue} Ahmad: (#1)}}
\newcommand{\ahmadm}[1]{\marginpar{{\small \color{blue}\textbf{#1}}}}
\newcommand{\setareh}[1]{{\color{blue} Setareh: #1}}

\newtheorem{definition}{Definition}
\newtheorem{theorem}{Theorem}
\newtheorem{lemma}{Lemma}
\newtheorem{remark}{Remark}
\newtheorem{proposition}{Proposition}
\newtheorem{assumption}{Assumption}
\newtheorem{example}{Example}
\newtheorem{corollary}{Corollary}
\newtheorem{proof}{Proof}

\newcommand{\V}{\mathbf{V}}
\newcommand{\U}{\mathbf{U}}
\newcommand{\X}{\mathbf{X}}
\newcommand{\Y}{\mathbf{Y}}
\newcommand{\y}{\mathbf{y}}
\newcommand{\Z}{\mathbf{Z}}
\newcommand{\bS}{\mathbf{S}} 
\newcommand{\F}{\mathbf{F}}
\newcommand{\A}{\mathbf{A}}
\newcommand{\R}{\mathbf{R}}
\newcommand{\Pa}{\mathbf{Pa}}
\newcommand{\N}{\mathbf{N}}

\newcommand{\bP}{\mathbb{P}}
\newcommand{\bR}{\mathbb{R}}
\newcommand{\bZ}{\mathbb{Z}}
\newcommand{\bF}{\mathbb{F}}
\newcommand{\bE}{\mathbb{E}}
\newcommand{\bv}{\mathbb{V}}

\newcommand{\cM}{\mathcal{M}}
\newcommand{\cN}{\mathcal{N}}
\newcommand{\cI}{\mathcal{I}}
\newcommand{\cJ}{\mathcal{J}}
\newcommand{\cG}{\mathcal{G}}
\newcommand{\cV}{\mathcal{V}}
\newcommand{\cU}{\mathcal{U}}
\newcommand{\cA}{\mathcal{A}}
\newcommand{\cX}{\mathcal{X}}
\newcommand{\cZ}{\mathcal{Z}}
\newcommand{\cS}{\mathcal{S}}
\newcommand{\cR}{\mathcal{R}}
\newcommand{\cP}{\mathcal{P}}
\newcommand{\cY}{\mathcal{Y}}
\newcommand{\cD}{\mathcal{D}}
\newcommand{\cK}{\mathcal{K}}
\newcommand{\cB}{\mathcal{B}}
\newcommand{\cQ}{\mathcal{Q}}


\newcommand{\sI}{\scalebox{.5}{\textbf{I}}}
\newcommand{\sG}{\scalebox{.5}{\textbf{G}}}
\newcommand{\cf}{\scalebox{.4}{\textbf{CF}}}
\newcommand{\CF}{{\textbf{\small CF}}}
\newcommand{\Plus}{\raisebox{.4\height}{\scalebox{.6}{+}}}
\newcommand{\sgn}{\mathrm{sign}}
\newcommand{\norm}[1]{\left\lVert#1\right\rVert}
\newcommand{\si}{\small{\textbf{+=}}}
\newcommand{\hi}{\small{\textbf{:=}}}
\newcommand{\pa}{\small{\textbf{pa}}}
\newcommand{\amax}{\mathrm{argmax}}
\newcommand{\cat}{\mathrm{cat}}
\newcommand{\con}{\mathrm{con}}
\newcommand{\ind}{\perp\!\!\!\!\perp} 
\newcommand{\CAP}{\textbf{\tiny CAP}}

\makeatletter
\newcommand*\bigcdot{\mathpalette\bigcdot@{.5}}
\newcommand*\bigcdot@[2]{\mathbin{\vcenter{\hbox{\scalebox{#2}{$\m@th#1\bullet$}}}}}
\makeatother
\newcommand{\pluseq}{\mathrel{+}=}
\newcommand{\abs}[1]{\left| #1 \right|}

\DeclarePairedDelimiter{\inner}{\langle}{\rangle}

\maketitle

\begin{abstract}
    As responsible AI gains importance in machine learning algorithms, properties such as fairness, adversarial robustness, and causality have received considerable attention in recent years. However, despite their individual significance, there remains a critical gap in simultaneously exploring and integrating these properties. In this paper, we propose a novel approach that examines the relationship between individual fairness, adversarial robustness, and structural causal models in heterogeneous data spaces, particularly when dealing with discrete sensitive attributes. We use causal structural models and sensitive attributes to create a fair metric and apply it to measure semantic similarity among individuals. By introducing a novel causal adversarial perturbation and applying adversarial training, we create a new regularizer that combines individual fairness, causality, and robustness in the classifier. Our method is evaluated on both real-world and synthetic datasets, demonstrating its effectiveness in achieving an accurate classifier that simultaneously exhibits fairness, adversarial robustness, and causal awareness.
\end{abstract}


\section{Introduction}
\label{sec:introduction}
In the ever-evolving landscape of machine learning, responsible AI has emerged as a pivotal focal point. Attributes such as fairness, adversarial robustness, and causality have taken center stage, each carrying its own weight in shaping ethical and socially reliable AI systems. Yet, the prevailing discourse often falls short of comprehensively addressing these dimensions in a unified manner, leaving a gap in our understanding of how they intersect and influence each other. 

Notably, within the realm of fairness, the scientific community has proposed various notions of fairness, broadly categorized as group fairness, examining model's performance across different demographic groups, and individual fairness, assessing model's performance on different individuals~\cite{pessach2022review, mehrabi2021survey}. While group fairness can guarantee similar classification performance on different demographic groups, it does not always guarantee individual fairness, i.e., that similarly qualified individuals,  receive similar outcomes~\cite{binns2020apparent}. 

Various formulations of individual fairness have been proposed in the literature, including Lipschitz~\cite{dwork2012fairness} and $\epsilon$-$\delta$~\cite{john2020verifying}. These formulations presume existence of a metric on the individuals that capture their (qualification) similarity. Such a similarity metric, by definition, is assumed to capture relevant features in the individuals that are important for the classification outcome, and to ignore features that should be irrelevant. For instance, in a hiring scenario, the similarity metric between the individuals could consider work experience and academic degree but should not take into account sensitive attributes. Due to this inherent fairness property in the definition of similarity metric, such metric is often referred to as a \emph{fair metric}. 
Various similarity functions have been proposed as fair metrics, including weighted $\ell_p$ norms, Mahalanobis distance, and feature embedding~\cite{benussi2022individual}.

In the domain of responsible AI, the study of causality is paramount, as the problems addressed often manipulate systems where inter-variable relations are governed by cause-and-effect mechanisms. In fact, many such sensitive attributes such as socio-economic status broadly affect the opportunities presented to individuals, which fair AI aims to rectify. Despite the introduction of causality as a critical lens in fairness literature~\cite{kusner2017counterfactual}, the aforementioned definitions of fair metrics, and the studied domains therein, have struggled to fully encompass the notion of robustness. While causal reasoning offers a foundation for addressing fairness, the inherent challenges of adversarial perturbations and their potential influence on fairness have remained largely unexplored.
 
In response to this gap, the initial step in our study is to propose a framework that defines a fair metric based on the functional structure of the underlying structural causal model. We propose a mathematical approach for protecting sensitive attributes by employing the concept of a pseudometric. Our proposed methodology enables the development of a fair metric that effectively mitigates bias across different levels of sensitive features in heterogeneous data spaces. Using our proposed fair metric we establish a causal adversarial perturbation (CAP) set to identify similar individuals. Subsequently, we analyze the characteristics of the CAP and its relationship with counterfactual fairness and adversarial robustness. Finally, we define a novel causal individual fairness notion based on the fair metric, which we refer to as \textbf{CAPI} fairness.

After formulating \textbf{CAPI} fairness, the next step is to train a classifier that guarantees this notion. This objective can be accomplished by applying bias mitigation methods during the in-processing stage. We ground our theoretical contributions in practicality by demonstrating the implementation of \textbf{CAPI} fairness within different classifiers and datasets. We initially examine the underlying cause of unfairness by defining the concept of \textit{unfair area}. We compute the unfair area for a linear model and design a \emph{post-processing} approach to obtain counterfactual fairness. Subsequently, to attain \textbf{CAPI} fairness which is a stronger notion, we employ adversarial learning techniques~\cite{madry2017towards} and present the first in-processing approach of \textbf{CAPI} fairness regularizer. To the best of our knowledge, this work is the first work that simultaneously addresses adversarial robustness, individual fairness, and causal structures in training a machine learning model. Our contributions are as follows:

\begin{itemize}
 \item \textbf{Causal Fair Metric (\S~\ref{sec:fair_metric}).} 
 Our primary contribution involves the establishment of a semi-latent space for the formulation of a fair metric. The introduction of this semi-latent space is essential to counteract the inherent bias embedded in the structural causal model. Achieving fairness necessitates the assurance that all potential interventions related to varying levels of sensitive attributes are considered. Based on this concept, we develop a fair metric that not only demonstrates effectiveness across diverse sensitive attributes but also incorporates the intricate aspects of the causal framework.
 \item \textbf{Causal Adversarial Perturbation (\S~\ref{sec:cap})} Building upon the foundation laid by our proposed causal fair metric, we introduce the concept of the causal adversarial perturbation. By leveraging the insights gained from our fair metric, causal adversarial perturbation emerges as a mechanism capable of capturing the similarity set in the presence of causal models.
  \item \textbf{CAPI Fairness (\S~\ref{sec:IF})}  Our third contribution entails the introduction of a novel fairness notion \textbf{CAPI} fairness. This concept emerges as a pivotal bridge that seamlessly connects individual fairness, adversarial robustness, and the underpinnings of causal structures. Furthermore, we establish a theoretical foundation for \textbf{CAPI} fairness, demonstrating its connections with counterfactual fairness and adversarial robustness.
  \item \textbf{Unfair Area (\S~\ref{sec:UA})} We further advance the discourse by defining the notion of the \emph{unfair area}, grounded within the context of \textbf{CAPI} fairness, and precisely explain this concept within the framework of a linear structural causal model and a classifier with a post-processing approach.
   \item \textbf{CAPI Fairness Classifier (\S~\ref{sec:CAL})} Our fifth contribution is the introduction of a pioneering in-processing adversarial learning method named \textbf{CAPIFY}. This method stands as the first of its kind to address \textbf{CAPI} fairness—simultaneously embodying individual fairness, adversarial robustness, and an awareness of causal dynamics.
   \item \textbf{Evaluation (\S~\ref{sec:CE})} We validate the efficacy of our approach through extensive evaluations on both real-world and synthetic datasets. These evaluations demonstrate the effectiveness of our proposed framework to simultaneously embody individual fairness, adversarial robustness, and causal awareness. 
\end{itemize}

\paragraph{Related Work.} Several studies have explored individual fairness by utilizing adversarial robustness techniques. 
\citet{doherty2023individual} investigated the association between adversarial robustness and $\epsilon$-$\delta$ individual fairness in Bayesian neural network inference. They considered a specified similarity metric and ensured that the network's output falls within a specified tolerance. 
\citet{benussi2022individual} introduce a method for certifying the $\epsilon$-$\delta$ individual fairness formulation in feed-forward neural networks. They define adversarial perturbation using $d_{\mathrm{fair}}$ and incorporate an adversarial regularizer in the training loss to achieve a balance between model accuracy and IF. 
\citet{xu2021robust} highlight that adversarial training may lead to notable discrepancies in both performance and robustness concerning group-level fairness. To address this issue, they propose a framework called fair robust learning that aims to enhance a model's robustness while ensuring fairness.
\citet{yeom2020individual} employed randomized smoothing techniques to ensure individual fairness in accordance with a specified weighted $\ell_p$ metric.
Several methods tackle individual fairness using Wasserstein distance and distributionally robust optimization \cite{yurochkin2019training, yurochkin2020sensei,vargo2021individually, pmlr-v115-jiang20a,pmlr-v108-jiang20a}. These approaches employ projected gradient descent and optimal transport with Wasserstein distance to optimize a model with perturbations that substantially modify the sensitive information within a specified distribution.
\citet{ruoss2020learning} introduced a mixed-integer linear programming approach to develop data representations that exhibit IF. These representations are designed to capture similarities among individuals by generating latent representations that remain unaffected by specific transformations of the input data.

Numerous prior studies \cite{grari2023adversarial, jung2019algorithmic, kim2018fairness, john2020verifying, adragna2020fairness, petersen2021post} have explored the connections among fairness, robustness, and causal structures individually or in pairs. However, to our knowledge, no previous research has explicitly examined the simultaneous interplay of all these properties.

\section{Preliminaries}
\label{sec:preliminaries}
\paragraph{Notation.}
In this study, random variables are indicated by boldface letters ($\V$), while regular lowercase letters ($v$) represent assignments or instances.
Matrices are denoted by bold uppercase letters, such as $\mathbf{F}$, with $[\mathbf{F}]_i$ referring to the $i$-th column vector of $\mathbf{F}$ and $[\mathbf{F}]_{i,j}$ representing the entry at row $i$ and column $j$ of $\mathbf{F}$.
The feature space $\mathcal{V}$ is constructed using $n$ random variables, denoted as $\V = (\V_1, \ldots, \V_n)$. 
\paragraph{Structural Causal Model (SCM).}
We make the assumption that feature variables $\mathcal{V}$ are generated by a SCM~\cite{pearl2009causality} denoted as $\mathcal{M}$, as described by a tuple $\langle \mathcal{G}, \V, \U, \mathbf{F}, \mathbf{P}_{\mathcal{U}} \rangle$.
Here, $\mathcal{G}$ represents a known directed acyclic graph (DAG), $\V=\{\V_i\}^{n}_{i =1}$ denotes a set of observed (indigenous) random variables, $\U=\{\U_i\}^{n}_{i =1}$ represents a set of noise (exogenous) random variables is assumed to be independent,
and $\bF$ is the set of \textbf{structural equations}, defined as $\bF = \{\V_i := f_i(\V_{\Pa(i)}, \U_i)\}_{i=1}^{n}$. These equations describe the causal relationship between each endogenous variable $\V_i$, its direct causes $\V_{\Pa(i)}$, and an exogenous variable $\U_i$ using deterministic functions $f_i$. Additionally, $\bP_{\U}$ represents the probability distribution over the exogenous variables.
The structural equations $\bF$ establish a mapping $\F: \cU \rightarrow \cV$ from exogenous to endogenous variables, along with an inverse image $\F^{-1}: \cV \rightarrow \cU$ that satisfies the property $\F\left(\F^{-1}(v)\right)=v$ for all $v\in\cV$.
The latent variable distribution entails a unique distribution $\bP(\V) = \prod_{i=1}^n \bP(\V_i \mid \V_{\Pa(i)})$ over the variables $\V$ \citep{peters2017elements}.
The marginal probability distribution of $\bP_{\V}$ with respect to the feature $\V_i$ is denoted as $\bP_{\V_i}$.
\paragraph{Additive Noise Model (ANM).}
In order to infer the unique causal structure $\cG$ from observational data $\mathcal{V}$, it is necessary to impose additional assumptions on the underlying SCM. 
One of the causally identifiable classes within SCMs is additive noise models~\cite{hoyer2009nonlinear}, which posit that the assignments follow the form: 
\begin{equation}
\label{eq:ANM}
\begin{aligned}
\bF = \{\textbf{V}_i \coloneqq f_i(\V_{\Pa(i)})+\textbf{U}_i \}_{i=1}^{n} \quad 
\implies  \\  \U = \V-f(\V) \implies  
\V = (I-f)^{-1}(\U)
\end{aligned}
\end{equation}
where $\U_i$ is an independent known distribution.
As observed in Eq.~\ref{eq:ANM}, obtaining $\U$ from $\V$ is straightforward, where $I$ represents the identity function ($I(v) = v$). Henceforth, we denote the inverse of $(I-f)^{-1}$ as $F$.
A specific class of ANMs is represented by \textbf{linear} SCMs, where the functions $f_i$ are assumed to be linear.
\paragraph{Counterfactuals.}
SCMs are employed to examine the effects of interventions, which entail external manipulations to modify the data generation process (Peters et al., 2017). Two primary types of interventions exist, hard interventions and soft interventions (see \S~\ref{sec:atc}). 
Interventions facilitate the examination of counterfactual statements $v^{\cf}$ for a given instance $v$ under hypothetical interventions on a variable. 
The counterfactual maps for hard interventions are denoted as $v_{\theta}^{\cf} := \CF(v, do(\V_{\cI} \hi \theta))=\F^{\theta}(\F^{-1}(v))$ where $\F^{\theta}$ is a simplified notation for $\F^{do(\V_{\cI} \hi\theta)}$.

\paragraph{Sensitive Attribute.}
A sensitive attribute, such as race, is an ethically or legally significant characteristic used in decision-making processes like hiring, lending, or criminal justice to determine fair treatment or outcomes for individuals or groups.
Let $\bS \in \{\V_1,\dots,\V_n\}$ be a sensitive attribute that has finite levels $\cS = \{s_1, \dots, s_k\}$.
For each instance $v$ of $\cV$, the set of \textbf{counterfactual twins} w.r.t protected variable $\bS$ is defined as
$\ddot{\bv} = \{ \ddot{v}_s = \CF(v,do(\bS \hi s)) :   s \in \cS \}$. 
\paragraph{Fairness.}
In fairness, a sensitive attribute defines a protected group, ensuring that machine learning models or algorithms do not disadvantage them. Researchers have proposed different notions of fairness, such as group fairness and individual fairness (IF)~\cite{tang2022and, le2022survey, mehrabi2021survey}.

Individual-level fairness, introduced by \citet{dwork2012fairness}, ensures that individuals who exhibit similarity according to predefined metrics are treated similarly with regard to outcomes. Various mathematical formulations have been proposed, including the Lipschitz Mapping-based formulation \cite{dwork2012fairness} and the $\epsilon$-$\delta$ formulation \cite{john2020verifying}. The classifier $h$ satisfies the $L$-Lipschitz IF condition when:
\begin{equation}
 d_{\cY}(h(v), h(w)) \leq L\ d_{\cV}(v, w)   \quad \forall v, w \in \cV
\end{equation}
where $d_{\cX}$ and $d_{\cY}$ represent metrics on the input and output spaces respectively, and $L \in \mathbb{R_+}$.
%
%

Counterfactual fairness, introduced by \citet{kusner2017counterfactual}, is another notion of individual-level fairness that deems a decision fair for an individual if it maintains consistency in both the real and a counterfactual scenario.
Formally, it can be expressed as:
\begin{equation}
 \mathbb{E}_{\mathbb{P}_{\mathbf{V}}}[\max_{s\in \mathbf{S}} d_{\mathcal{Y}}(h(\mathbf{V}), h(\ddot{\mathbf{V}}_s))] \leq \epsilon 
\end{equation}
\paragraph{Adversarially Robust Learning.}
Adversarially robust learning aims to create algorithms and models that can withstand adversarial attacks, which involve purposeful perturbations or modifications to input data to induce misclassification or misleading predictions
\cite{goodfellow2014explaining, madry2017towards}.
In this framework, models are trained considering the most challenging perturbations of the data rather than the original data itself:
\begin{equation}
\label{eq:ARL}
\min_\psi {\bE}_{(v,y) \sim \cP_{\mathcal{D}}}[ \max_{\delta \in B_{\Delta}(v)} \ell(h_\psi(v+\delta),y)]
\end{equation}
where, $B_{\Delta}(v)$ is the set of perturbations for the instance $v$, $\cP_{\mathcal{D}}$ is observation distribution, $\ell$ is the classification loss function, and $\psi$ are the weights of the classifier.  

\section{Causal Fair Metric }
\label{sec:metric_fairness}
Achieving individual fairness necessitates the formulation of a fair metric, which, in pursuit of this goal, gives rise to two primary challenges. Firstly, the presence of diverse feature types within the SCM, such as categorical or continuous attributes, introduces complexities stemming from its heterogeneous nature. Secondly, inherent biases may be encoded within the SCM, thereby necessitating that our classifier comprehends the full spectrum of hypothetical interventions applied to instances relative to the levels of sensitive attributes. These twin focal points constitute the primary focus of the ensuing chapter.

\subsection{Fair Metric}
\label{sec:fair_metric}

When dealing with independent features, constructing similarity functions based on their attributes and aggregating them through a product metric is relatively straightforward. However, in the context of a causal structure, the integration of causality into metric formulation becomes pivotal. To tackle this, instances undergo a transformation into an independent space where a metric is established. This established metric is subsequently employed in defining a similarity function within the original feature space via the push-forward metric technique.

In the presence of SCM and a sensitive attribute, the similarity function $d$ should be robust to twins and slight perturbations of non-sensitive features. This means that $d$ should not significantly change after a hard intervention ($do(\bS \coloneqq s)$) with respect to the levels of $S$, or after an additive intervention on continuous features.
In ANMs, a hard intervention removes the causal structure of $\bS$ and is equivalent to setting $f_i$ to zero and fixing $\U_i \coloneqq s$. Moreover, additive intervention is equivalent to adding $\delta$ to $\U_i$ while keeping $f_i$ unchanged. Consequently, the latent space changes during the hard intervention, replacing the sensitive latent variable $\U_i$ with $\bS$ following the distribution $\bP_{\bS}$. This motivates the definition of a \textit{semi-latent space}.
\begin{definition}[Semi-latent Space]
Consider SCM $\mathcal{M}$ with sensitive features indexed by $I$. We define the semi-latent space $\cQ$ as a combination of observed sensitive features $\V_i$ with distribution $\bP_{\V_i}$ where $i \in I$, and latent variables $\U_j$ for other features with distribution $\bP_{\U_j}$. 
\end{definition}
Let $v = (v_1, v_2, \dots, v_n)$ be an instance in the observed space and $u = (u_1, u_2, \dots, u_n) = F^{-1}(v)$ be the corresponding instance in the latent space. The mapping $T: \cV \rightarrow \cQ$ transforms $v$ to the semi-latent space $q = (q_1, q_2, \dots, q_n) = T(v)$, where $q_i$ is defined as follows:
\begin{equation}
    q_i\coloneqq
    \begin{cases}
        v_i & i \in I \\
        u_i & i \notin I
    \end{cases}
\end{equation}
The inverse function $v = T^{-1}(q)$ is determined as follows:
\begin{equation}
    v_i\coloneqq
    \begin{cases}
        q_i & i \in I \\
        f_i(v_{\pa(i)}) + q_i & i \notin I
    \end{cases}
\end{equation}
The identity $v = T^{-1}(T(v))$ holds straightforwardly.

The semi-latent space allows us to describe the counterfactual of instance $v$ w.r.t. hard action $do(\V_I \hi \theta)$:
    \begin{equation}
    \label{eq:doing}
        \mathbf{CF}(v, do(\V_I \hi \theta)) = T^{-1}(T(v)\odot_{I} \theta)
    \end{equation}
Here, $v\odot_{I} \theta$ represents a masking operator that modifies the values of $I$ entries in vector $v$ by replacing $\theta$. 

In the semi-latent space, a causal structure-independent similarity function can be readily established. Let $(\cU_{i}, d_{\cU_i})$ denote the metric space for the latent space corresponding to $\V_i$. For sensitive variables $\bS_i$, $(\cS_i, d_{\cS_i})$ is considered a pseudometric or metric space. Thus, the semi-latent space $(\cQ , d_{\cQ})$ has a metric obtained as the product of metrics. 
To establish a fair metric, incorporating sensitive features into the similarity function is crucial. We adopt the approach by \citet{ehyaei2023robustness}, treating the protected feature as a pseudometric.
\begin{definition}[Pseudometric Protected~\cite{ehyaei2023robustness}]
\label{def:protected}
In SCM $\mathcal{M}$, suppose the sensitive feature $\bS$ endowed with a pseudometric space $(\cS,d_{\cS})$. $\bS$ is partially protected if there are two levels with zero distance:
\begin{equation}
\exists s,s' \in \mathcal{S} \quad  \text{s.t.} \quad   d_{\mathcal{S}}(s, s') =0 \quad  \land \quad s \neq s'
\end{equation}
If for all $s,s' \in \cS$ we have $d_{\cS}(s, s') =0$, then $\bS$ is called protected feature.
\end{definition}

\begin{figure*}[ht]
    \centering
    \includegraphics[width=1\textwidth]{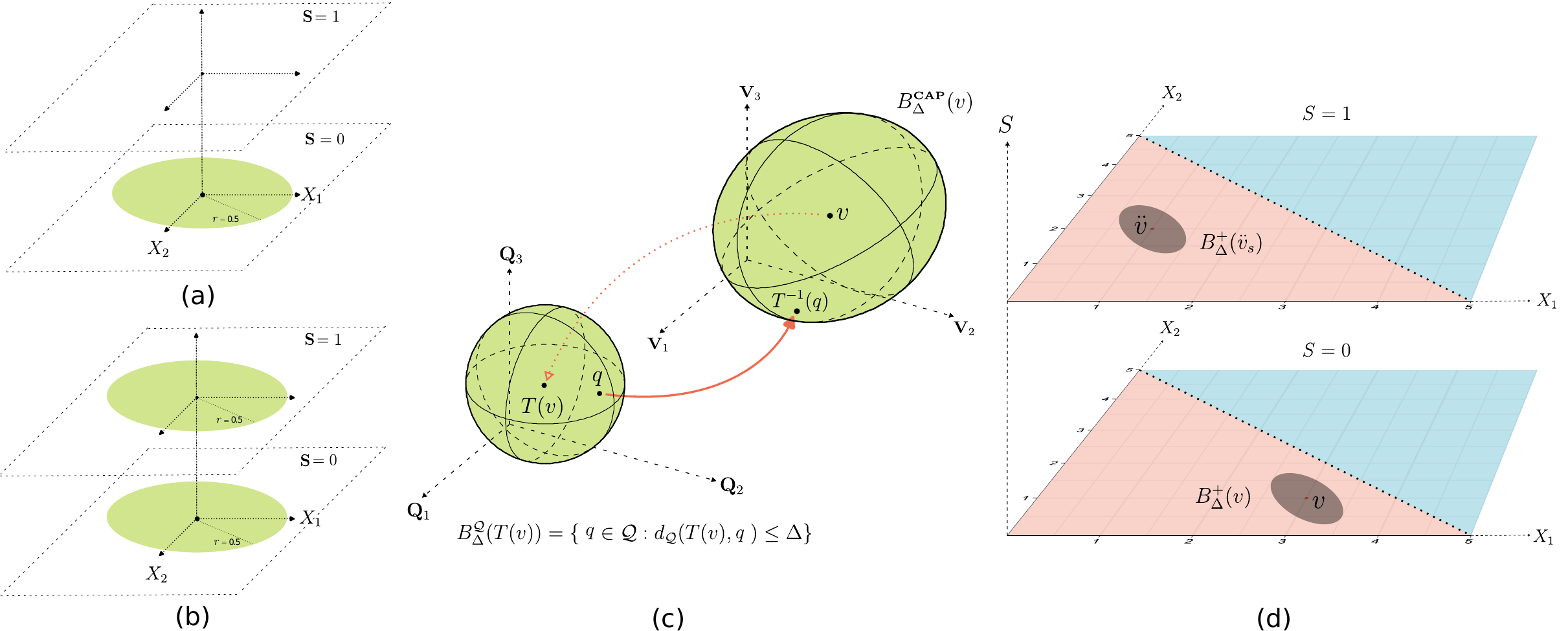}
    \caption{The difference in unit ball shape between considering the sensitive attribute as a Euclidean metric (a) and as a trivial pseudometric (b). The geometric interpretation of CAP is mapping a closed unit ball in semi-latent space (c). Causal adversarial perturbation is the union of continuous perturbations around each twin (d).}
    \label{fig:CAP}
\end{figure*}
By employing the pseudometric for sensitive attributes within the semi-latent space metric, a fair metric can be established in the feature space using the push-forward metric:
\begin{equation}
\label{eq:fairmetric}
 d_{\textit{fair}}(v, w) = d_{\cQ}(T(v), T(w))  
\end{equation}
fair metric enables us to define small perturbations of factual values to identify similar instances.
\subsection{Causal Adversarial Perturbation}
\label{sec:cap}
Adversarial perturbation involves the manipulation of input data to evaluate the resilience of machine learning models. The introduction of a fair metric contributes to the definition of adversarial perturbation in alignment with causal relationships.
\begin{definition}[Causal Adversarial Perturbation]
\label{def:cap}
Let $\mathcal{M}$ be an SCM with sensitive attributes, and $d_{\textit{fair}}$ be its fair metric. The CAP for instance $v$ is defined as:
\begin{equation}
B^{\CAP}_{\Delta}(v) = \{ w \in \mathcal{V} :  d_{\textit{fair}}(v,w)\leq \Delta\}
\end{equation}
where $\Delta \in \mathbb{R}_{\geq 0}$.
CAP can be seen as transforming the unit ball in the semi-latent space using the inverse mapping function $T^{-1}$:
\begin{equation}
B_{\Delta}^{\cQ}(q)=\{p \in \cQ:d_{\cQ}(q,p) \leq \Delta\}.
\end{equation}
then $B^{\CAP}_{\Delta}(v) = T^{-1}(B_{\Delta}^{\cQ}(T(v)))$.
\end{definition}
\begin{remark}
\label{rm:simpleACP}
When all features are continuous or all sensitive features don't have parents, CAP simplifies interpretation. In these cases, the semi-latent space coincides with the latent space, and CAP is achieved by transforming the unit ball in the latent space using the mapping function $F$. Specifically, $B^{\CAP}_{\Delta}(v) = F(B_{\Delta}^{\cU}(F^{-1}(v)))$, where $B_{\Delta}^{\cU}$ represents a closed ball with radius $\Delta$ in the latent space.
\end{remark}
Building upon Remark \ref{rm:simpleACP}, we seek a concise geometric interpretation of CAP
by perturbing only the continuous feature of the SCM. Let $q = (z,x) \in \cQ$ with $x$ as the continuous part and $z$ as the categorical part of features. We define $B^{\cQ_{\Plus}}_{\Delta}$ as the unit ball with a radius of $\Delta$, specifically designed for the continuous part:
\begin{equation*}
B^{\cQ_{\Plus}}_{\Delta}(q)=\{q' = (z',x') \in \cQ: z' = z \quad \land \quad d_{\cX}(x',x) \leq \Delta\}
\end{equation*}
Without loss of generality, assuming a norm on the continuous part, we define the closed unit disk as $D^{\mathcal{X}}_{\Delta} = \{\delta: \| \delta\|\ \leq \Delta \quad  \land \quad  \delta_{|\mathcal{Z}} = 0\}$ where $\mathcal{Z}$ is categorical part of feature space.
Thus, in this scenario, $B^{\cQ_{\Plus}}_{\Delta}$ is derived from:
\begin{equation}
 B^{\cQ_{\Plus}}_{\Delta}(q) = \{q + \delta: \delta_{|\mathcal{X}} \in D^{\mathcal{X}}_{\Delta} \}   
\end{equation}
By defining $B^{\Plus}_{\Delta}(v) = T^{-1}(B^{\cQ_{\Plus}}_{\Delta}(T(v)))$, the CAP can be decomposed into $B^{\Plus}_{\Delta}$, as stated in the following proposition.
\begin{proposition}
\label{pr:ballShape}
Let $B^{\CAP}_{\Delta}(v)$ represent the CAP around instance $v = (z,x)$ with radius $\Delta$, and let $\Theta_{\Delta} = \{ \theta \in \cZ: (\theta,.) \in B^{\cQ}_{\Delta}(T(v)) \}$ denote the set of categorical levels within the perturbation ball. The counterfactual perturbation can be expressed as:
\begin{equation}
\label{eq:union}
B^{\CAP}_{\Delta}(v) = \bigcup_{\theta \in \Theta_{\Delta}} B^{\Plus}_{\Delta_{\theta}}(\CF(v, \theta))
\end{equation}
where $\Delta_{\theta}$ represents the value of the continuous part of $\Delta$. For instance, in the case of using the $L_2$ product metric, $\Delta_{\theta} = \sqrt{\Delta^2- d_{\cZ}(\theta, s)^2}$.
\end{proposition}
The decomposition of perturbation allows analyzing the shape of CAP for a sensitive attribute, especially for small $\Delta$ values. This aspect is elaborated upon in the subsequent corollary.
\begin{corollary}
\label{pr:decompose}
If $\bS$ is a protected feature and other categorical variables in $\cM$ are not partially protected, there exists a $\Delta_0$ such that for all $\Delta \leq \Delta_0$:
\begin{equation}
\label{eq:decompose}
B^{\CAP}_{\Delta}(v) = \bigcup_{s \in \mathcal{S}} B^{\Plus}_{\Delta}(\ddot{v}_s)
\end{equation}
Consequently, for all $v, w \in \ddot{\mathbb{V}}$, we have $B^{\CAP}_{\Delta}(v) = B^{\CAP}_{\Delta}(w)$.
\end{corollary}
The CAP definition considers causal similarity in relation to counterfactuals. The subsequent lemma shows that a CAP with a diameter $0$ represents the set of twins.
\begin{corollary}
	\label{pr:twinDef}
    If $\bS$ is a protected feature and other categorical variables in $\mathcal{M}$ are not partially protected, the counterfactual twins correspond to the zero-radius CAP:
	$$\ddot{\mathbb{V}} = B^{\CAP}_{0}(v) \coloneqq \lim_{\Delta \rightarrow 0} B^{\CAP}_{\Delta}(v)$$
\end{corollary}
\subsection{CAPI Fairness}
\label{sec:IF}
This section presents our innovative concept of causal individual fairness, denoted as \textbf{CAPI} fairness. Within the Lipschitz formulation of IF, we introduce the metric $d_{\textit{fair}}$ as a measure in the feature space:
$$d_{\mathcal{Y}}(h(v), h(w)) \leq d_{\textit{fair}}(v, w)$$
Here, $d_{\mathcal{Y}}$ represents the metric applied in the outcome space. By incorporating a fair metric as a similarity function, individual fairness now encompasses both the causal structure and the sensitive protected feature.
\begin{proposition}
\label{pr:notions}
CAPI Fairness implies both Counterfactual Fairness and Adversarial Robustness:
\begin{align*}
\textbf{CAPI Fairness} \Rightarrow \textbf{Counterfactual Fairness} \\
\textbf{CAPI Fairness} \Rightarrow \textbf{Adversarial Robustness}
\end{align*}
However, the inverse statements are not necessarily true.
\end{proposition}

\section{Fair Classifier}
\label{sec:fair_classifier}
In this section, we will initially explore the origins of unfairness concerning IF in the context of an SCM. Following that, we will introduce IF classifiers based on CAPI fairness.	
\subsection{Unfair Area}
\label{sec:UA}
To analyze the bottlenecks in designing fair classifiers, we should understand the origins of unfairness.
We begin by defining unfair areas for CAPI fairness, inspired by \citet{ehyaei2023robustness}.
\begin{definition}[Unfair Area]
Let $\mathcal{M}$ denote an SCM, $\Delta$ diameter of CAP, and $h$ be a binary classifier operating on $\V$. The unfair area includes instances where the CAPI fairness property is not met:
\begin{equation}
	A_{\Delta}^{\neq} := \{v \in \cV: \exists v' \in B^{\CAP}_{\Delta}(v) \quad \text{s.t.} \quad h(v) \neq h(v')\}
\end{equation}
\end{definition}
To understand the shape of the unfair area, we aim to determine $A_{\Delta}^{\neq}$ assuming linear SCMs and classifiers (see Fig.~\ref{fig:II}).
\begin{figure*}[h]
    \centering
    \includegraphics[width=0.4\textwidth]{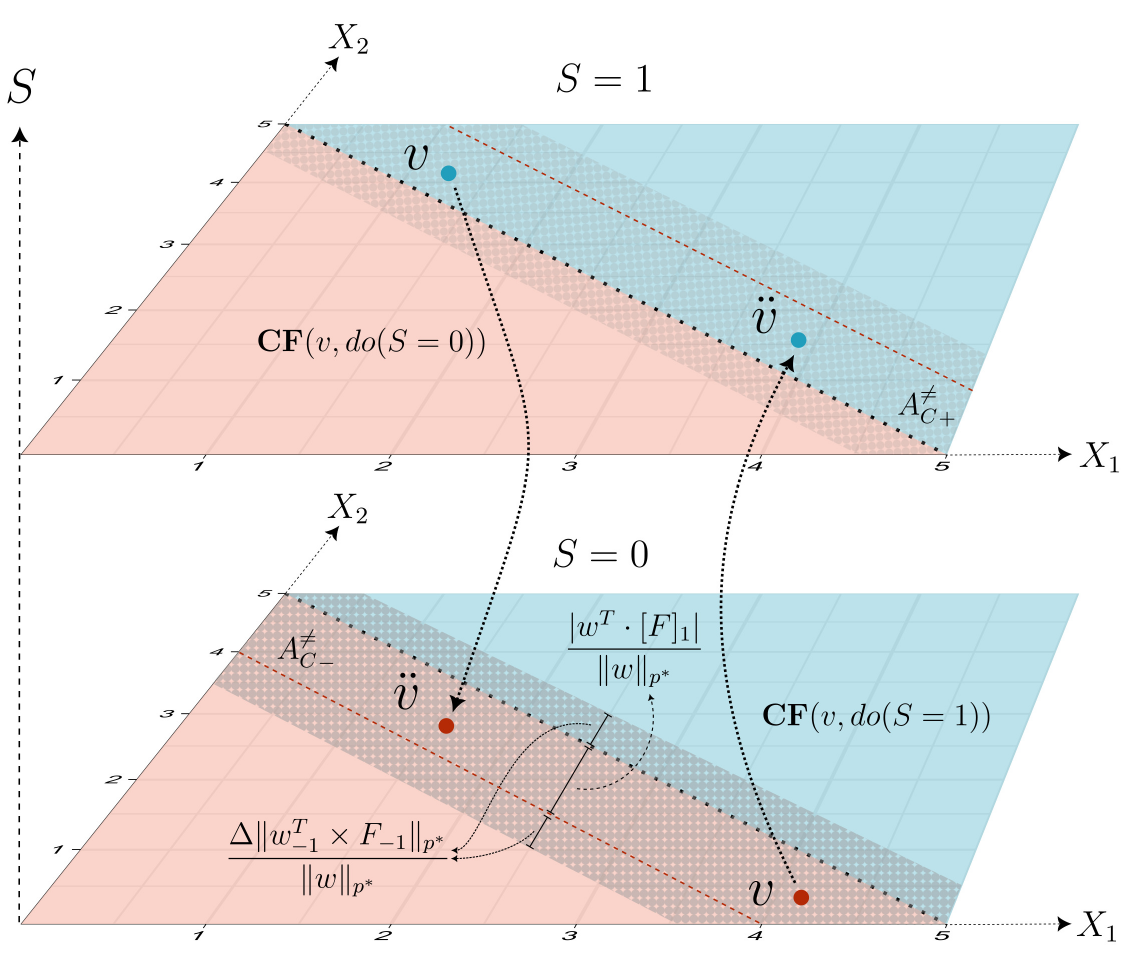}
    \hfill
    \includegraphics[width=0.4\textwidth]{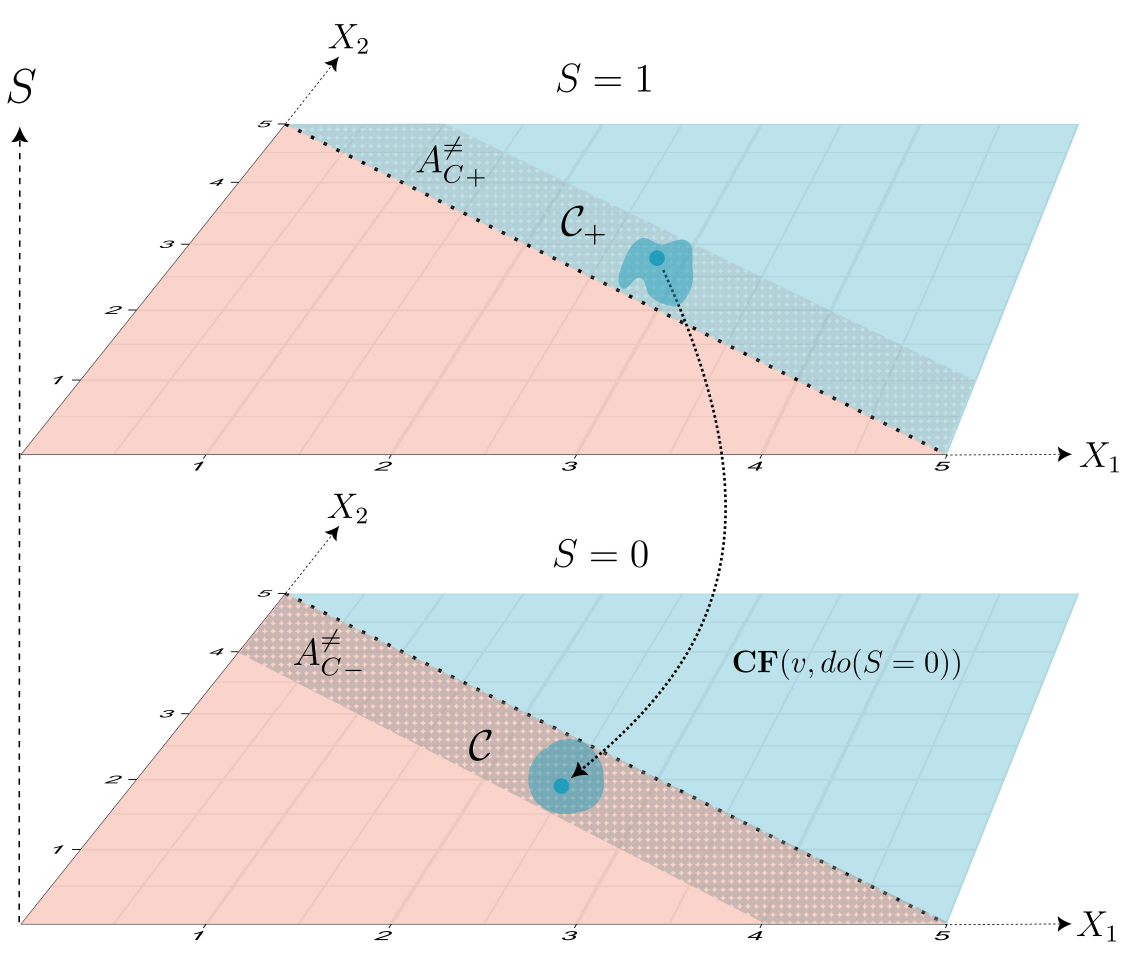}
    \caption{The unfair area for linear SCM and classifier consists of two parts: counterfactual and adversarial robustness (left) The counterfactual fairness mitigation idea is based on the property that the twin of the  twin is equal to the instance $v$ (right).}
    \label{fig:II}
\end{figure*}
\begin{proposition}
\label{pr:unfairArea}
Consider a linear SCM with a binary linear classifier $h(v) = \text{sign}(w^T \cdot v - b)$, where $w \in \mathbb{R}^n$. 
Assume $\mathcal{M}$ has one binary sensitive attribute $S \in \{0,1\}$ and other features $X$ are continuous. Without loss of generality, let $V_1$ represent the sensitive attribute. The unfair area for counterfactual fairness is delineated as follows:
\begin{equation*}
\label{eq:unfairCEq}
\begin{aligned}
A_{C}^{\neq} = \{v = (s, x)\in \mathcal{V} : & \quad \mathrm{sign}((s-(1-s))*h(v)) \geq 0 \quad \land \\ & dist(x,L) \leq   \frac{|w^T\bigcdot [F]_1|}{\|w\|_{p^*}} \}
\end{aligned}
\end{equation*}
The unfair area $A_{\Delta}^{\neq}$ is defined as the band parallel to the classifier boundary $L$:
\begin{equation*}
\label{eq:unfairEq}
A_{\Delta}^{\neq} = \{v = (s, x)\in \mathcal{V} : dist(x,A_{C}^{\neq}) \leq   \frac{\Delta\|w^T_{-1} \times F_{-1} \|_{p^*}}{\|w\|_{p^*}} \}
\end{equation*}
Here, $L$ denotes the decision boundary of the classifier, while $w_{-1}$ and $F_{-1}$ represent the continuous components of $w$ and $F$, respectively.
\end{proposition}

According to Prop.~\ref{pr:unfairArea}, a straightforward condition can be derived for ensuring counterfactual fairness.
\begin{corollary}
	\label{pr:linfair}
    Considering the condition in Prop.~\ref{pr:unfairArea}, achieving a counterfactually fair classifier for $\mathcal{M}$ is impossible unless $F$ and $w$ satisfy the equation $w^T \cdot [F]_{1} = 0$. This implies that the classifier $h$ relies solely on a subset of variables that are non-descendants of $\bS$ in $\mathcal{M}$.
\end{corollary}
In assessing CAPI unfairness, a meaningful indicator is the probability associated with the unfair area in the trained classifier. 
\begin{definition}[Unfair Area Indicator (UAI)]
\label{def:uai}
Let $\mathcal{M}$ be the SCM with parameters denoted by $\langle \mathcal{G}, \V, \U, \mathbb{F}, \mathbb{P}_{\U}\rangle$, and $\hat{h}$ be the trained binary classifier. 
The probability $\mathbb{P}_{\mathcal{V}}(A_{\Delta}^{\neq})$, referred to as the Unfair Area Indicator, quantifies the likelihood of the CAPI unfairness for $\hat{h}$.
\end{definition}
Taking into account the concept of unfair areas and the inherent property that a twin's twin is an identity function, we present a post-processing technique involving label-flipping. This method aims to mitigate counterfactual fairness issues.
\begin{proposition}[Counterfactual Unfairness Mitigation]
\label{pr:cum}
Let ${A_{C}^{\neq}}_+$ and ${A_{C}^{\neq}}_-$ represent the positive and negative regions of counterfactual unfairness $A_{C}^{\neq}$, respectively (see Fig.~\ref{fig:II}). Assuming $\mathcal{C} \subset {A_{C}^{\neq}}_-$, the unfair area mitigation method involves flipping the labels of instances in $\mathcal{C}$ to positive. By changing labels, the reduction in unfairness area is given by:
\begin{equation}
\mathbb{P}_{\mathcal{V}}(A_{C}^{\neq})-\mathbb{P}_{\mathcal{V}}(\mathcal{C})-\mathbb{P}_{\mathcal{V}}(\mathcal{C}_+)
\end{equation}
Here, $\mathcal{C}_+$ is a subset of ${A_{\Delta}^{\neq}}_+$, representing the points in ${A_{\Delta}^{\neq}}_+$ whose corresponding twins belong to the set $\mathcal{C}$. If we set $\mathcal{C} = {A_{C}^{\neq}}_-$, complete mitigation of counterfactual fairness can be achieved.
\end{proposition}
\begin{remark}
The label-flipping direction ($+ \ to \ -$) does not inherently impact counterfactual unfairness mitigation. However, fairness considerations often involve a preferred direction. In such cases, flipping the sign of the unfair region in relation to this preferred direction can be employed to promote fairness.
\end{remark}
Label flipping alone is insufficient to remove CAPI unfairness. Therefore, in the next section, we introduce an additional in-processing method to mitigate unfairness.
\subsection{Causal Adversarial Learning}
\label{sec:CAL}
Fair adversarial learning aims to achieve high accuracy in predicting the target variable while ensuring fairness regarding sensitive attributes. This involves formulating a min-max optimization problem, where the model simultaneously minimizes the classification error and maximizes the adversarial loss~\cite{pessach2022review}. 
In previous chapters, the concept of CAP was discussed. Now, we formulate the objective function for Causal Adversarial Learning (CAL). Let $\mathcal{D} = \{(v_i, y_i)\}_{i = 1}^{n}$ represent the set of observations. The objective function to be minimized over the classifier space in CAL is as follows:
\begin{equation}
\label{eq:optimization}
\min_\psi {\bE}_{(v,y) \sim \cP_{\mathcal{D}}}[ \max_{w \in B^{\CAP}_{\Delta}(v)} \ell(h_\psi(w),y)]
\end{equation}
The optimization objective in Eq.~\ref{eq:optimization} promotes the proximity of values for $h$ within the neighborhood $B^{\CAP}_{\Delta}(v)$ to $h(v)$.
According to Lem.~\ref{pr:estimate}, we can establish the inequality
$f(v + \delta) \leq f(v) +  |\delta^T \nabla_v f(v)| + \gamma(\Delta, v)$.
By setting $f(v + \delta) = \ell(h(T^{-1}(T(v)+\delta),y)$, we can utilize Cor.~\ref{pr:decompose} to represent the expression within the expectation of Eq.~\ref{eq:optimization} as follows:
\begin{equation}
\label{eq:regularizer}
\begin{aligned}    
    & \max_{w \in B^{\CAP}_{\Delta}(v)}\ell(h(w),y) = 
    \max_{s \in \mathcal{S}} \max_{w \in B^{\Plus}_{\Delta}(\ddot{v}_s)}\ell(h(w), y) = \\
    & \max_{s \in \mathcal{S}} \max_{\delta \in D^{\mathcal{X}}_{\Delta}} \ell(h(T^{-1}(T(\ddot{v}_s)+\delta),y) \leq \\
    & \max_{s \in \mathcal{S}} \max_{\delta \in D^{\mathcal{X}}_{\Delta}} \ell(h(T^{-1}(T(\ddot{v}_s)),y) 
    + |\delta^T \nabla^{\mathcal{X}}_{\ddot{v}_s} f(\ddot{v}_s)| + \gamma(\Delta, \ddot{v}_s)
    \leq \\
    & \ell(h(v),y) + \max_{s \in \mathcal{S}}  \ell(h(\ddot{v}_s), y)
    + \\ & \max_{s \in \mathcal{S}} \max_{\delta \in D^{\mathcal{X}}_{\Delta}} |\delta^T \nabla^{\mathcal{X}}_{\ddot{v}_s} f(\ddot{v}_s)| + \gamma(\Delta, \ddot{v}_s) = \\
    & \ell(h(v),y) + \underset{II}{\underbrace{\max_{s \in \mathcal{S}}  \ell(h(\ddot{v}_s), y)}}
    + \max_{s \in \mathcal{S}} \underset{II}{\underbrace{(\| \nabla^{\mathcal{X}}_{\ddot{v}_s} f(\ddot{v}_s)\|_* + \gamma(\Delta, \ddot{v}_s))}}
\end{aligned}
\end{equation}
The symbol $\nabla^{\mathcal{X}}$ denotes the gradient operator for continuous features. The validity of the final equation can be established by bounding $|\delta^T \nabla^{\mathcal{X}}_{\ddot{v}_s}f(\ddot{v}_s)|$ using the dual norm $\|\nabla^{\mathcal{X}}_{\ddot{v}_s} f(\ddot{v}_s)\|_*$.

The adversarial loss function, as per Eq.~\ref{eq:regularizer}, comprises a regular loss function and a regularizer, which can be decomposed into two components. The first component addresses counterfactual fairness by capturing the discrepancy between the instance $y$ and the corresponding twins' classifier label. The second component measures the adversarial robustness of classifier $h$ regarding the continuous features surrounding each twin. 
Assuming random observations, the evaluation of the robustness property is narrowed down to the instance denoted as $v$. Hence, the reformulated expression for the regularizer can be stated as follows:

\begin{equation}
\label{eq:modifiedReg}
\begin{aligned}    
\mathcal{R}(v) =  \quad & \mu_1 * \max_{s \in \mathcal{S}} \ell(h(\ddot{v}_s),y) \quad + \\
& \mu_2 *\gamma(\Delta, v) + \mu_3 *\| \nabla^{\mathcal{X}}_{v} f(v)\|_* 
\end{aligned}
\end{equation}
where the hyperparameters $\mu_i \in \mathbb{R}$ determine the extent of regularization in the model.

\section{Numerical Experiments}
\label{sec:CE}
\begin{table*}[ht]
\centering
\fontsize{7}{8}\selectfont
\begin{tabular}{ l  c@{\hspace{4pt}} c@{\hspace{4pt}} c@{\hspace{4pt}} c@{\hspace{4pt}} | c@{\hspace{4pt}} c@{\hspace{4pt}} c@{\hspace{4pt}} c@{\hspace{4pt}} | c@{\hspace{4pt}} c@{\hspace{4pt}} c@{\hspace{4pt}} c@{\hspace{4pt}} | c@{\hspace{4pt}} c@{\hspace{4pt}} c@{\hspace{4pt}} c@{\hspace{4pt}} | c@{\hspace{4pt}} c@{\hspace{4pt}} c@{\hspace{4pt}} c@{\hspace{4pt}} | c@{\hspace{4pt}} c@{\hspace{4pt}} c@{\hspace{4pt}} c@{\hspace{4pt}}}
	\toprule
    & \multicolumn{8}{c}{\textbf{Real-World Data}} 
	& \multicolumn{16}{c}{\textbf{Synthetic Data}}
    \\
	\cmidrule(lr){2-9}  \cmidrule(lr){10-25} 
	& \multicolumn{4}{c}{\textbf{Adult}} 
	& \multicolumn{4}{c}{\textbf{COMPAS}}
	& \multicolumn{4}{c}{\textbf{IMF}}
	& \multicolumn{4}{c}{\textbf{LIN}}
    & \multicolumn{4}{c}{\textbf{Loan}}
    & \multicolumn{4}{c}{\textbf{NLM}}
	\\
	\cmidrule(lr){2-5} \cmidrule(lr){6-9} \cmidrule(lr){10-13} \cmidrule(lr){14-17} \cmidrule(lr){18-21} \cmidrule(lr){22-25} 
	\multirow{1}{*}{\textbf{Trainer}} & $A$
	& $U_{.05}$
	& $CF$
    & $\cR_{.05}$
	& $A$
	& $U_{.05}$
	& $CF$
    & $\cR_{.05}$
	& $A$
	& $U_{.05}$
	& $CF$
    & $\cR_{.05}$
	& $A$
	& $U_{.05}$
	& $CF$
    & $\cR_{.05}$
	& $A$
	& $U_{.05}$
	& $CF$
    & $\cR_{.05}$
	& $A$
	& $U_{.05}$
	& $CF$
    & $\cR_{.05}$ 
	\\
	\midrule
 AL & \textbf{0.80} & 0.22 & 0.18 & \textbf{0.04} & \textbf{0.68} & 0.18 & 0.14 & 0.04 & 0.63 & 0.30 & 0.28 & 0.11 & 0.59 & 0.90 & 0.90 & 0.26 & 0.81 & 0.27 & 0.27 & 0.16 & 0.57 & 0.55 & 0.53 & 0.37\\
CAL & 0.80 & 0.23 & 0.18 & 0.05 & 0.67 & 0.14 & 0.10 & 0.04 & 0.59 & 0.35 & 0.34 & 0.13 & 0.59 & 0.90 & 0.90 & 0.26 & 0.67 & 0.26 & 0.26 & 0.19 & 0.34 & 0.48 & 0.46 & 0.24\\
\textbf{CAPIFY} & 0.78 & \textbf{0.07} & \textbf{0.03} & 0.05 & 0.63 & \textbf{0.06} & \textbf{0.01} & 0.04 & \textbf{0.72} & \textbf{0.04} & \textbf{0.00} & 0.04 & 0.62 & \textbf{0.19} & \textbf{0.15} & \textbf{0.09} & 0.70 & \textbf{0.22} & \textbf{0.22} & 0.14 & 0.45 & \textbf{0.38} & \textbf{0.36} & 0.24\\
ERM & 0.80 & 0.21 & 0.18 & 0.04 & 0.68 & 0.21 & 0.17 & 0.04 & 0.72 & 0.05 & 0.02 & \textbf{0.04} & 0.68 & 0.44 & 0.43 & 0.18 & \textbf{0.85} & 0.35 & 0.35 & 0.23 & \textbf{0.71} & 0.57 & 0.54 & 0.41\\
LLR & 0.76 & 0.23 & 0.19 & 0.04 & 0.62 & 0.35 & 0.32 & \textbf{0.04} & 0.69 & 0.22 & 0.20 & 0.08 & 0.64 & 0.31 & 0.31 & 0.27 & 0.69 & 0.27 & 0.26 & \textbf{0.12} & 0.44 & 0.41 & 0.39 & \textbf{0.20}\\
ROSS & 0.68 & 0.22 & 0.08 & 0.11 & 0.62 & 0.22 & 0.16 & 0.06 & 0.72 & 0.05 & 0.02 & 0.04 & \textbf{0.69} & 0.47 & 0.44 & 0.11 & 0.83 & 0.38 & 0.38 & 0.27 & 0.70 & 0.55 & 0.52 & 0.43\\
	\midrule
\end{tabular}
\caption{
The table displays the outcomes of our numerical experiment, wherein different trainers are compared based on their input sets in terms of accuracy (A, higher values are better), CAPI fairness metrics ($U_{.05}$, lower values are better), Counterfactual Unfair area ($CF$, lower values are better), and the non-robust percentage concerning adversarial perturbation with radii $0.05$ ($\mathcal{R}_{.05}$, lower values are better). The best-performing techniques for each trainer, dataset, and metric are indicated in bold. The findings highlight that CAPIFY outperforms other trainers in reducing CAPI unfairness. The standard deviation average for CAPIFY is 0.028, whereas for the other methods, it is 0.038.}
\label{tab:sim}
\end{table*}

\begin{figure*}[ht]
\centering
\includegraphics[width=0.95\textwidth]{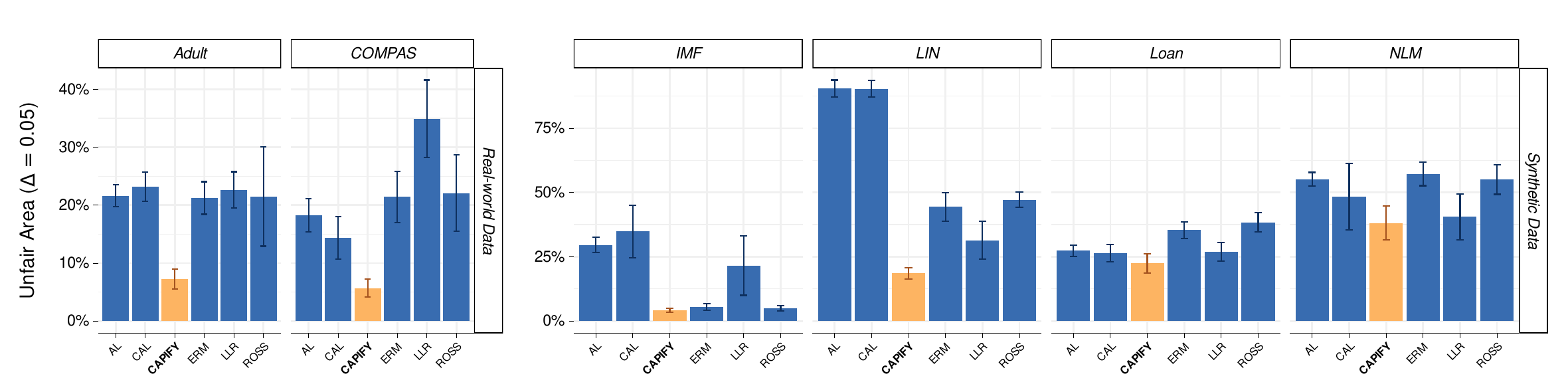}
\includegraphics[width=0.95\textwidth]{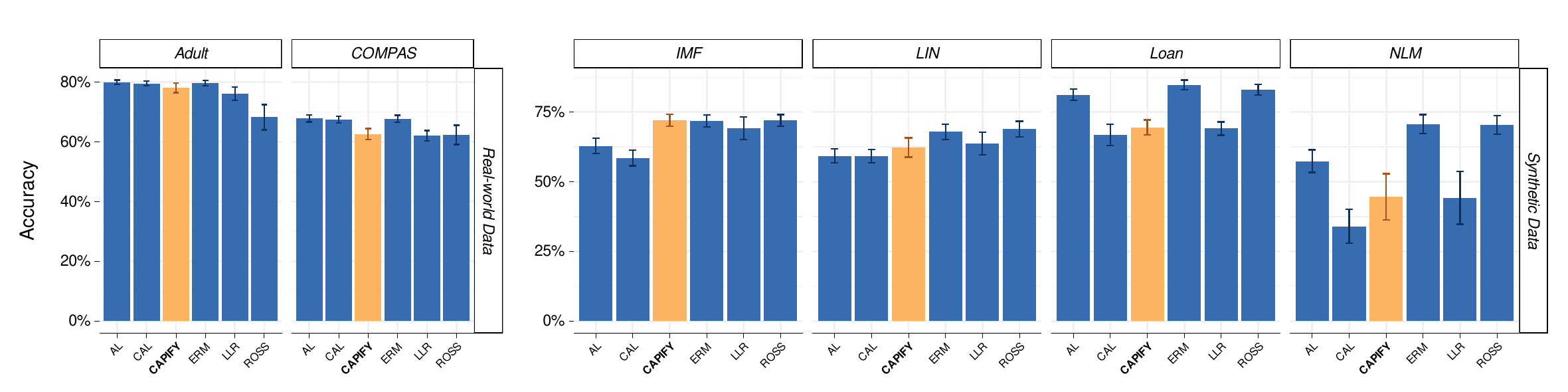}
\caption{Figure depicts our numerical experiment's results, showcasing diverse trainers and datasets to evaluate CAPIFY performance. The initial bar plot assesses trainer performance through UAI values (favoring lower values) at $\Delta = .05$. The subsequent bar plot contrasts methods based on prediction performance (favoring higher values). }
\label{fig:com_study}
\end{figure*}

In this study, we empirically validate the theoretical propositions presented in the paper. We assess the performance of the CAPIFY and CAL training methods in comparison to conventional empirical risk minimization (ERM) and other pertinent techniques, including Adversarial Learning (AL) \cite{madry2017towards}, Locally Linear Regularizer (LLR) training \cite{qin2019adversarial}, and Ross method \cite{ross2021learning}. Our experimentation involves real datasets, specifically Adult \cite{kohavi1996uci} and COMPAS \cite{washington2018argue}, which are pre-processed according to \cite{dominguez2022adversarial}. Furthermore, we consider three synthetic datasets related to Linear (LIN), Non-linear (NLM), and independent futures (IMF) SCMs, along with the semi-synthetic Loan dataset based on \cite{karimi2020algorithmic}.

We utilize a multi-layer perceptron with three hidden layers, each comprising 100 nodes, for the COMPAS, Adult, NLM, and Loan datasets. Logistic regression is employed for the remaining datasets. To evaluate classifier performance, we measure accuracy and Matthews correlation coefficient (MCC). Furthermore, we quantify CAPI fairness using UAI across various $\Delta$ values, including 0.05, 0.01, and 0.0. Additionally, we compute UAI for non-sensitive scenarios, employing $\Delta$ values of 0.05 and 0.01 to represent the non-robust data percentage. Additional comprehensive details about the computational experiments are available in the $\S$~\ref{app:simulation}.

We performed our experiment using 100 different seeds, and the results are presented in Tables \ref{tab:sim} and \ref{tab:sim2}. Figures \ref{fig:com_study} and \ref{fig:com_study_2} illustrate that the CAPIFY method exhibits a lower unfair area ($U_\Delta$) for $\Delta = 0.05$, $\Delta = 0.01$, and $\Delta = 0.0$. 
However, the CAL method shows unsatisfactory accuracy due to the issues reported previously \cite{qin2019adversarial}.
Compared to ERM, CAPIFY shows slightly lower accuracy, a trade-off noted in multiple studies \cite{pessach2022review}. Notably, real-world data indicates a greater reduction in unfairness than in accuracy.
Moreover, CAPIFY exhibits robustness and counterfactual fairness attributes (see Tab.~\ref{tab:sim}), making it the favored model when assessing both concepts. For more results, see $\S$~\ref{app:simulation}.

\section{Discussion and Future Work}
\label{sec:discussion}
In this study, we introduce a comprehensive method considering individual fairness (IF) and robustness within an underlying causal model. We establish adversarial learning through the use of CAP. Remarkably, our CAP strategy sets itself apart by not requiring assumptions for all categorical features, a departure from the approach by \citet{ehyaei2023robustness}. Our CAP framework exclusively focuses on sensitive features.

In this study we use the discrete sensitive features for simplicity, every theoretical and numerical part are satisfied for continuous sensitive attribute as well.
Our approach avoids specific assumptions for defining IF based on the $L$-Lipschitz formulation. Instead, we can reframe everything using the $\epsilon$-$\delta$ formulation.
The optimization in Eq.~\ref{eq:optimization} may yield nonlinear decision boundaries, particularly with numerous features. To tackle this, we adopt the locally linear regularizer (LLR) proposed by \citet{qin2019adversarial}. LLR is advantageous in deep learning for countering overfitting, enhancing generalization through smoother function learning, and attaining leading computational performance.

An objection to our approach is that, like many fair learning methods, although we address unfairness by introducing a regularizer, there's no assured theoretical guarantee for the resulting classifier to uphold individual fairness. In future research, our goal is to develop a classifier with theoretical foundations that endorse CAPI fairness principles. 
\bibliography{aaai24}
\clearpage
\section{Appendix}
\label{sec:appendix}
%
%
\subsection{Glossary Table}
\tablefirsthead{\toprule Symbol&\multicolumn{1}{c}{Notion} \\ \midrule}
\tablehead{%
\multicolumn{2}{c}%
{{\bfseries  Continued from previous column}} \\
\toprule
Symbol&\multicolumn{1}{c}{Notion}\\ \midrule}
\begin{supertabular}{p{3cm} p{5cm}}
    SCM &  structural causal model \\
    CAP & causal adversarial perturbation \\
    IF & individual fairness \\
    CAPI & CAP-based individual fairness \\
	$\V$ & random variable \\
    $v$ & instance of random variable $\V$\\
    $[\mathbf{F}]_i$ & referring to the $i$-th column vector of matrix $\mathbf{F}$ \\
    $[\mathbf{F}]_{i,j}$ & representing the entry at row $i$ and column $j$ of $\mathbf{F}$ \\
    $\mathcal{V}$ & feature space \\
    $\mathcal{M}$ & structural causal model \\
    $\mathcal{G}$ & directed acyclic graph\\
    $\V=\{\V_i\}^{n}_{i =1}$ & observed (indigenous) random variables\\
    $\U=\{\U_i\}^{n}_{i =1}$ & represents a set of noise (exogenous) random variables \\
    $\pa(i)$ & parents of feature $\V_i$ w.r.t. DAG \\ 
    $\bF$ &  the set of structural equations of SCM\\
    $\bP_{\U}$ & represents the probability distribution over the exogenous variables \\
    $\bP_{\V}$ & entailed distribution over  endogenous space\\
    $\F: \cU \rightarrow \cV$ & reduced-form mapping from exogenous to endogenous space \\
    $\F^{-1}: \cV \rightarrow \cU$ & inverse image of $\F$ function \\
    $v^{\cf}$ & counterfactual statements for a given instance $v$ \\
    $do(\V_{\cI} \hi \theta))$ & do-operator corresponds to hard intervention \\
    $v_{\theta}^{\cf}$ & counterfactual w.r.t. hard interventions \\
    $\CF(v, do(\V_{\cI} \hi \theta))$ & counterfactual w.r.t. $do(\V_{\cI} \hi \theta))$ \\
    $\bS$ & sensitive attribute \\
    $\cS$ & the levels set of sensitive attribute \\
    $\ddot{v}_s$ & the twins of $v$ is obtained by $\CF(v,do(\bS \hi s))$\\
    $\ddot{\bv}$ & the set of all twins of $v$\\
    $d_{\cX}$ & metric on the feature space \\
    $d_{\cY}$ & metric on the label space \\
    $\cP_{\mathcal{D}}$ & observational probability distribution \\
    $B_{\Delta}(v)$ & perturbations ball for the instance $v$ \\
    $\ell$ & the classification loss function \\
    $\psi$ & the weights of the parametric classifiers\\
    $\cQ$ & semi-latent space \\
    $T: \cV \rightarrow \cQ$ & mapping from feature space to semi-latent space\\
    $v\odot_{I} \theta$ & represents a masking operator that modifies the values of $I$ entries in vector $v$ by replacing $\theta$ \\
    $(\cU_{i}, d_{\cU_i})$ &  metric space for the latent space\\
    $(\cS_i, d_{\cS_i})$ & the pseudometric or metric space for sensitive feature\\
    $(\cQ ,d_{\cQ})$ & product metric on $\cQ$\\
    $d_{\textit{fair}}$ & fair metric \\
    $\Delta$ & perturbation radius \\
    $B^{\CAP}_{\Delta}(v)$ & causal adversarial perturbation around $v$ with radii $\Delta$\\
    $B_{\Delta}^{\cQ}(q)$ & the closed ball with radii $\Delta$ in semi-latent space\\
    $\Z$ & categorical part of features \\
    $B^{\cQ_{\Plus}}_{\Delta}(q)$ & the closed ball within $\mathcal{Q}$, limited to modifications in continuous features \\
    $B^{\Plus}_{\Delta}(v)$ & the CAP around $v$ that the only continuous feature is perturbed \\
    $\Theta_{\Delta}$ & the set of categorical levels inside $B^{\CAP}_{\Delta}(v)$ \\
    $\Delta_{\theta}$ & represents the value of the continuous part of $\Delta$\\
    $A_{\Delta}^{\neq}$ & unfair area with radii $\Delta$\\
    $A_{C}^{\neq}$ & unfair area corresponds to counterfactual fairness\\
    ${A_{C}^{\neq}}_+$ & the positive regions of counterfactual fairness \\
    ${A_{C}^{\neq}}_-$ & the negative regions of counterfactual fairness \\
    $\mathcal{D} = \{(v_i, y_i)\}_{i = 1}^{n}$ & the set of observations \\
\end{supertabular}

\label{sec:gt}
%
%
%
%
\subsection{Theoretical Background}
\label{sec:atc}
\begin{definition}[Hard and Soft Intervention]
With Hard interventions (denoted by the \textbf{do}-operator notation $\cM^{do(\V_{\cI} \hi \theta)}$), a subset $\cI\subseteq [n]$ of feature values $\V_{\cI}$ is forcibly set to a constant $\theta\in\bR^{|\cI|}$ by excluding specific components of the structural equations:
\begin{equation*}
\label{eq:hard}
\bF^{do(\V_{\cI} \hi \theta)}= 
\begin{cases}
\textbf{V}_i := \theta_i & \text{if} \ i \in \cI \\
\textbf{V}_i := f_i (\V_{\Pa(i)}, \U_i )& \text{otherwise}
\end{cases}
\end{equation*}
Hard interventions disrupt the causal relationship between the affected variables and all of their ancestors in the causal graph, while preserving the existing causal relationships. In contrast, soft interventions maintain all causal relationships while modifying the structural equation functions.
For example, \textbf{additive (shift) interventions} \cite{eberhardt2007interventions} with symbol $\cM^{do(\V_{\cI} \si \delta)}$, 
\footnote{In the causality literature, the do-operator is only applied to hard interventions. In this work, to avoid using more notation, we use the $do(\V_{\cI} \si \delta)$ for additive interventions too.}
were changed the features $\V$ by some perturbation vector $\delta\in\bR^{n}$: 
\begin{equation*}
\label{eq:soft}
\bF^{do(\V_{\cI} \si \delta)} = \left\{V_i := f_i\left(\V_{\Pa(i)}, \U_i\right)+\delta_i\right\}_{i=1}^{n}.
\end{equation*}
\end{definition}
\begin{definition}[$\epsilon$-$\delta$ Individual Fairness] 
 Let us consider $\epsilon \geq 0$, $\delta \geq 0$, and a mapping $h: \mathcal{V} \rightarrow \mathcal{Y}$. Individual fairness is said to be satisfied by $h$ if:
 \begin{equation*}
 \forall v, w \in \cV  \quad d_{\cV}(v, w) \leq \delta \quad \Longrightarrow \quad d_{\cY}(h(v), h(w)) \leq   \epsilon 
\end{equation*}
where $d_{\cX}$ and $d_{\cY}$ represent metrics on the input and output spaces respectively, and $L \in \mathbb{R_+}$.
\end{definition}
\begin{definition}[pseudometric space] 
\label{def:pseudo}
A pseudometric space $(X,d)$ is a set $X$ together with a non-negative real-valued function $d:X\times X\longrightarrow \mathbb {R} _{\geq 0}$, called a pseudometric, such that for every $x,y,z \in X$,
\begin{itemize}
\item $d(x,x)=0.$
\item Symmetry: $d(x,y)=d(y,x)$
\item Triangle inequality: $d(x,z)\leq d(x,y)+d(y,z)$
\end{itemize}
The trivial example of pseudometric $d(x,y)=0$ for all $x,y\in X$
\end{definition}

\begin{definition}[Middle Intervention \cite{ehyaei2023robustness}]
\label{def:midint}
Consider SCM $\mathcal{M}$, with $n$ and $m$ continuous and categorical variables. 
Let the indexes of categorical  and continuous variables in vector $V$ be $\mathcal{I}_{\text{cat}}$  and $\mathcal{J}_{\text{con}}$ ($\mathcal{I}_{\text{cat}} \cup \mathcal{J}_{\text{con}} =\{1,2, \dots, n+m\}$).
The middle intervention $\mathcal{M}^{\theta_{\mathcal{I}}, \delta_{\mathcal{J}}}$  fixes the values of a subset $\mathcal{I} \subset \mathcal{I}_{\text{cat}}$ of categorical features $\mathbf{V}_{\mathcal{I}}$ to some fixed $\theta_{\mathcal{I}} \in\mathcal{Z}^{|\mathcal{I}|}$ and additive intervened the continuous features $\mathbf{V}_{\mathcal{J}}$ of a subset $\mathcal{J} \subset \mathcal{J}_{\text{con}}$ by some $\delta_{\mathcal{J}} \in \mathbb{R}^{|\mathcal{J}|}$ while preserving all other causal relationships.
The description of the structural equations for $\mathbb{S}_{\mathcal{M}}^{\delta_{\mathcal{I}},\theta_{\mathcal{J}}}$ follows:
\begin{equation}
\mathbf{V}_i := 
\begin{cases}
\theta_i & \text{if} \ i \in \mathcal{I} \\
f_i (\mathbf{V}_{\text{pa}(i)}, \mathbf{U}_i)+\delta_i, &  \text{if} \ i \in \mathcal{J}  \\
f_i (\mathbf{V}_{\text{pa}(i)}, \mathbf{U}_i) & \text{otherwise}
\end{cases}
\end{equation}
Similar to other interventions, the middle intervention's counterfactual is defined as 
$\mathbf{CF}\left(v, \theta_{\mathcal{I}}, \delta_{\mathcal{J}}; \mathcal{M}\right) = \mathbf{S}_{\mathcal{M}}^{\theta_{\mathcal{I}}, \delta_{\mathcal{J}}}(\mathbf{S}_{\mathcal{M}}^{-1}(v))$. 
\end{definition}

%
%
\section{Proofs}
\begin{lemma}
\label{pr:estimate}
    Consider a function $f:\mathbb{R}^n \rightarrow \mathbb{R}$ that is once-differentiable and a local neighborhood defined by $B(\Delta) = \{ \delta \in \mathbb{R}^n: \| \delta\| \leq \Delta\}$. Then for all $\delta \in B(\Delta)$:
    \begin{equation}
        |f(v + \delta) - f(v)| \leq |\delta^T \nabla_v f(v)| + \gamma(\Delta, v)
    \end{equation}
    where $\displaystyle \gamma (\Delta, v) = \max_{\delta \in B(\Delta)} |f(v + \delta)  - f(v) -\delta^T \nabla_v f(v)|$.
\end{lemma}
\begin{proof}
Firstly, let us express $\abs{\ell(x+\delta) - \ell(x)}$ as follows:
\begin{align*}
|f(x+\delta) - f(x)| &= \\|\delta^T \nabla_x f(x) + &f(x+\delta) -f(x) -\delta^T \nabla_x f(x)|.   
\end{align*}
By the triangle inequality, we can establish the following bound:

$$\abs{f(x+\delta) - f(x)} \leq \abs{\delta^T \nabla_x f(x)} + g(\delta; x)$$,

where $g(\delta; x) = \abs{f(x+\delta) -f(x) -\delta^T \nabla_x f(x)}$. Notably, since
$\gamma(\epsilon, x) = \max_{\delta \in B(\epsilon)} g(\delta; x),$

it follows that for all $\delta \in B(\epsilon)$,
$$
\abs{f(x+\delta) - f(x)} \leq \abs{\delta^T \nabla_x f(x)} + \gamma(\epsilon, x).
$$
\end{proof}

\subsection*{Proposition~\ref{pr:ballShape}}

Writing the definitions yields the proof directly.
By the Def. \ref{def:cap}, the $B^{\CAP}_{\Delta}(v)$ is equal to:
\begin{equation*}
\begin{aligned}
& B^{\CAP}_{\Delta}(v)  =   \\
& T^{-1}(B_{\Delta}^{\cQ}(T(v))) =
T^{-1}(\{ q = (z, x) \in B_{\Delta}^{\cQ}(T(v)) \}) =  \\
& T^{-1}(\big\{ \bigcup_{\theta \in \Theta_{\Delta}} (\theta, x) \in B_{\Delta}^{\cQ}(T(v))\big\}) =  \\
& \bigcup_{\theta \in \Theta_{\Delta}}  \big\{T^{-1}( (\theta, x) \in B_{\Delta}^{\cQ}(T(v))\big\}) =  \\
& \bigcup_{\theta \in \Theta_{\Delta}}  T^{-1}(\big\{ (\theta, x):   d_{\cQ}((\theta, x),(T(v)_{|_{cat}},T(v)_{|_{con}}) \leq \Delta\big\}) = \\
& \bigcup_{\theta \in \Theta_{\Delta}}  T^{-1}(\big\{ (\theta, x):   d_{\cX}( x,T(v)_{|_{con}}) \leq \Delta_{\theta}\big\}) = \\
& \bigcup_{\theta \in \Theta_{\Delta}}  T^{-1}(\{ (\theta, T(v)_{|_{con}}) + \delta: \delta_{|\mathcal{X}} \in D^{\mathcal{X}}_{\Delta_{\theta}} \land \delta_{|\mathcal{Z}} = 0 \} ) = \\
& \bigcup_{\theta \in \Theta_{\Delta}}  T^{-1}(\{ T(\CF(v, \theta)) + \delta: \delta_{|\mathcal{X}} \in D^{\mathcal{X}}_{\Delta_{\theta}} \land \delta_{|\mathcal{Z}} = 0 \} ) = \\
& \bigcup_{\theta \in \Theta_{\Delta}} B^{\Plus}_{\Delta_{\theta}}(\CF(v, \theta)) 
\end{aligned}
\end{equation*}
In the above, let $\Delta_{\theta}$ represent the value of the continuous part of $\Delta$. 
The proof is completed by verifying the equation $T(\CF(v, \theta)) = (\theta, T(v)_{|_{con}}) = T(v)\odot_{I} \theta$ by Eq. \ref{eq:doing}, which is correct by definition.

\subsection*{Corollary~\ref{pr:decompose}}

    Let $q = (z, x) \in B_{\Delta}^{\cQ}(T(v))$ represent a point in the semi-latent space, and $\mathcal{S} = \{s_1, s_2, \dots, s_k \}$ denote the levels of the sensitive feature $\bS$. Considering that $\bS$ is the only protected feature, and other categorical variables are not partially protected, and taking into account that categorical variables have a discrete topology, we can find a value $\Delta_0$ such that for $\Delta \leq \Delta_0$, the set $\Theta_{\Delta} = \{ z' \in \mathcal{Z} : (z',.) \in B_{\Delta}^{\cQ}(T(v)) \}$ contains $z'$-values where all categorical variables, except $\bS$, have fixed values. To find $\Delta_0$, it is sufficient to consider
    \begin{equation*}
        \Delta_0 = \min_{z,z' \in \mathcal{Z}}d_{\mathcal{Z}}(z,z') \quad \text{subject to} \quad \Delta_0 > 0
    \end{equation*}
    where $\mathcal{Z}$ is the set of all levels of categorical features. Since $\mathcal{Z}$ is a finite set, $\Delta_0$ exists.
    Thus, we can assume that $\mathcal{M}$ has only one categorical variable, denoted as $\bS$. As a result, $\Theta_{\Delta} = \{s_1, s_2, \dots, s_k \}$.
    Furthermore, since $\bS$ is a sensitive feature, $d_{\mathcal{S}}(s,s') = 0$, which implies $\Delta_{\theta} = \Delta$. By utilizing Proposition~\ref{pr:ballShape}, we can write:
    \begin{equation*}
    \begin{aligned}
        B^{\CAP}_{\Delta}(v) = & \bigcup_{\theta \in \Theta_{\Delta}} B^{\Plus}_{\Delta_{\theta}}(\CF(v, \theta)) \\
        & \bigcup_{s \in \cS} B^{\Plus}_{\Delta}(\CF(v, s)) = \bigcup_{s \in \mathcal{S}} B^{\Plus}_{\Delta}(\ddot{v}_s)
    \end{aligned}
    \end{equation*}
    This completes the proof.

\subsection*{Corollary~\ref{pr:twinDef}}

By employing Corollary~\ref{pr:decompose}, we can express the result as follows:
\begin{equation*}
\begin{aligned}
    B^{\CAP}_{0}(v)  & = \lim_{\Delta \rightarrow 0} B^{\CAP}_{\Delta}(v) = 
    \lim_{\Delta \rightarrow 0} \bigcup_{s \in \mathcal{S}} B^{\Plus}_{\Delta}(\ddot{v}_s) = \\
    &\bigcup_{s \in \cS}  \lim_{\Delta \rightarrow 0} T^{-1}(\{ T(\ddot{v}_s) + \delta: \delta_{|\mathcal{X}} \in D^{\mathcal{X}}_{\Delta} \land \delta_{|\mathcal{Z}} = 0 \} ) = \\
    &\bigcup_{s \in \cS}  \lim_{\Delta \rightarrow 0} T^{-1}(\{ T(\ddot{v}_s) \} ) = 
    \bigcup_{s \in \mathcal{S}} \ddot{v}_s  
\end{aligned}
\end{equation*}
The last equation holds true because as $\Delta$ approaches zero, $\delta$ also approaches zero.

\subsection*{Proposition~\ref{pr:notions}}

    As shown in Proposition~\ref{pr:twinDef}, the CAP set contains the twins of each instance $v$. If we have individual fairness for classifier $h$, it implies that for all $s \in \cS$, we have $d_{\cY}(h(v), h(\ddot{v}_s)) \leq d_{\cV}(v,\ddot{v}_s)$. Considering that, by definition, the distance $d_{\cV}(v,\ddot{v}_s) = 0$, it follows that for each twin, $d_{\cY}(h(v), h(\ddot{v}_s))$ is also zero. Therefore, individual fairness guarantees counterfactual fairness.

    To disprove the inverse of the proposition, consider a linear SCM, denoted as $\cM$, consisting of one sensitive attribute $\bS$ and two manipulable continuous features $\X_1$ and $\X_2$. The model is characterized by the following structural equation and classifier:
    \begin{equation*}
    \begin{cases}
    S := S_A, & U_A \sim \mathcal{B}(0.5) \\
    X_1 := U_1, & U_1 \sim \mathcal{N}(0,1) \\
    X_2 := X_1 + U_2, & U_2 \sim \mathcal{N}(0,1) \\
    h(v) = \text{sign}(\omega \cdot v - b), & \omega = (0,1,1), \quad b = 5
    \end{cases}
    \end{equation*}
    Here, $\mathcal{B}(p)$ represents the Bernoulli distribution with probability $p$, and $\mathcal{N}(\mu,\sigma^2)$ is the normal distribution with mean $\mu$ and variance $\sigma^2$.
    In this specific example, it is observed that the sensitive attribute does not act as a parent to other features within $\mathcal{M}$. Furthermore, the classifier is entirely independent of the sensitive attribute, implying that $h$ exhibits counterfactual fairness with respect to $\bS$.

    To prove that CAPI fairness implies adversarial robustness, we rely on Corollary~\ref{pr:decompose}. Given CAPI fairness on $B^{\CAP}_{\Delta}(v)$, we can infer its validity for $B^{\Plus}_{\Delta}(v)$. Consequently, we obtain the property $d_{\cY}(h(v), h(w)) \leq d_{\cV}(v,w)$. Since $B^{\Plus}_{\Delta}(v)$ represents an adversarial perturbation around $v$, the adversarial robustness property follows.

    The converse assertion of the non-implantation of IF in the context of adversarial robustness is straightforward.

\subsection*{Proposition~\ref{pr:unfairArea}}

    To prove our claim, we relied on the theorems presented in the paper by \citet{ehyaei2023robustness}. According to Corollary~\ref{pr:decompose}, the unfair area comprises points where the classifier labels differ for instances within the continuous perturbation around twins. First, we try to find the unfair area with respect to twins of each instance.
    As stated in Proposition 3.5 of \citet{ehyaei2023robustness}, the construction of the unfair area for counterfactual fairness is as follows: 
    \begin{equation*}
    \begin{aligned}
    A_{C}^{\neq} = \{v = (s, x)\in \mathcal{V} : & \mathrm{sign}((s-(1-s))*h(v)) \geq 0 \quad \land \\ & dist(x,L) \leq   \frac{|w^T\bigcdot [F]_1|}{\|w\|_{p^*}} \}
    \end{aligned}
    \end{equation*}  
    The CAPI unfair area is determined by buffering the counterfactual unfair area using the radius of the balls $B^{\Plus}_{\Delta}(\ddot{v}_s)$ for all $s \in \cS$. As the SCM is linear, these radii are equal, and according to Proposition 3.7 of \citet{ehyaei2023robustness}, they can be expressed as $\frac{\Delta|w^T{-1} \times F{-1} |{p^*}}{|w|{p^*}}$. This equation validates our claim that the unfair area is constructed by: 
    \begin{equation*}
A_{\Delta}^{\neq} = \{v = (s, x)\in \mathcal{V} : dist(x,A_{C}^{\neq}) \leq   \frac{\Delta\|w^T_{-1} \times F_{-1} \|_{p^*}}{\|w\|_{p^*}} \}
\end{equation*}

\subsection*{Corollary~\ref{pr:linfair}}

When $w^T \cdot [F]_{1} = 0$, Proposition~\ref{pr:unfairArea} ensures that the classifier $h$ is individually fair for any instance of $v$. Conversely, if the classifier $h$ is individually fair, then for all $s' \in \mathcal{S}$, according to Corollary 3.3 of \citet{ehyaei2023robustness}, we have:
\begin{equation*}
\begin{aligned}
 h(\ddot{v}_{s'}) = & \text{sign}(w^T \cdot (v+ (s' - s) \cdot [F]_{1}) - b) = \\ &
\text{sign}(w^T \cdot v + (s' - s) \cdot w^T \cdot [F]_{1} - b) = h(v) \Rightarrow \\  
& |w^T \cdot v + (s' - s) \cdot w^T \cdot [F]_{1} - b| = \text{constant} \\ & \forall a' \in \mathcal{A} \Rightarrow (s' - s) \cdot w^T \cdot [F]_{1} = 0.
\end{aligned}
\end{equation*}
The final equation concludes the proof.

\subsection*{Lemma~\ref{pr:cum}}

According to Proposition \ref{pr:unfairArea}, the extent of counterfactual unfairness corresponds to $\mathbb{P}_{\mathcal{V}}(A_{C}^{\neq})$. If we apply a post-processing method and invert the labels of a subset $\mathcal{C} \subset {A_{C}^{\neq}}_-$, the unfair area is reduced not only by $\mathbb{P}_{\mathcal{V}}(\mathcal{C})$, but also by incorporating the area of $\mathcal{C}_+$ into the fair region due to the similar label sign of its counterfactual. Hence, the mitigation of unfair area is given by:
\begin{equation*}
\mathbb{P}_{\mathcal{V}}(A_{C}^{\neq}) - \mathbb{P}_{\mathcal{V}}(\mathcal{C}) - \mathbb{P}_{\mathcal{V}}(\mathcal{C}_+)
\end{equation*}

%
%
%
%
\section{Example}
\label{sec:example}
For the purpose of providing a clearer understanding of the definitions and theoretical aspects, we employ a simple example. We consider a linear SCM denoted as $\cM$, comprising one sensitive attribute $\bS$ and two manipulable continuous features $\X_1$ and $\X_2$. The model is specified by the following structural equations and classifier:
\begin{equation*}
\label{eq:lin_model}
\begin{cases}
S := U_S, & U_S \sim \mathcal{B}(0.5) \\
X_1 := 2S + U_1, & U_1\sim \mathcal{N}(0,1) \\
X_2 := S-X_1 + U_2, & U_2 \sim \mathcal{N}(0,1) \\
h(v) = \text{sign}(\omega \cdot v -b), & \omega = (0,1,1), \quad b = 5
\end{cases}
\end{equation*}
where $\mathcal{B}(p)$ represents a Bernoulli distribution with probability $p$, and $\mathcal{N}(\mu,\sigma^2)$ denotes a normal distribution with mean $\mu$ and variance $\sigma^2$. In the continuous part, we utilize the Euclidean metric, while in the semi-latent space, we adopt the $L_2$ product metric. To investigate the impact of the protected pseudometric, we consider both the trivial pseudometric and the Euclidean metric for the sensitive feature. The results of the CAP are presented in Figure \ref{fig:CAP}.
%
%
\section{Numerical Experiments Details}
\label{app:simulation}
%
%
\subsection{Synthetic Data Models}

The structural equations used to generate the SCMs in \S~\ref{sec:CE} are listed below.
For the LIN, NLM and IMF SCMs, we generate the protected feature $\bS$ and variables $\X_i$ according to the following structural equations:
\begin{itemize}
    \item linear SCM (LIN): 
\begin{equation*}
\label{eq:lin_model}
\begin{cases}
S := U_S, & 			U_S \sim \mathcal{B}(0.5) 	\\
X_1 := 2S + U_1, & 	U_1\sim \mathcal{N}(0,1)	\\
X_2 := S-X_1 + U_2, & 			U_2 \sim \mathcal{N}(0,1) \\
Y \sim \mathcal{B}((1 + exp(-(X_1 + X_2))^{-1})
\end{cases}
\end{equation*}
    \item Non-linear Model (NLM)
\begin{equation*}
\begin{cases}
S := U_S, & 			U_S \sim \mathcal{B}(0.5) 	\\
X_1 := 2S^2 + U_1, & 	U_1\sim \mathcal{N}(0,1)	\\
X_2 := S-X_1^2 + U_2, & 			U_2 \sim \mathcal{N}(0,1) \\
Y \sim \mathcal{B}((1 + exp(-(X_1 + X_2)^2)^{-1})
\end{cases}
\end{equation*}
    \item Independent Manipulable Feature (IMF)
\begin{equation*}
\begin{cases}
S := U_S, & 			U_S \sim \mathcal{B}(0.5) 	\\
X_1 := U_1, & 	U_1\sim \mathcal{N}(0,1)	\\
X_2 := U_2, & 			U_2 \sim \mathcal{N}(0,1) \\
Y \sim \mathcal{B}((1 + exp(-(X_1 + X_2))^{-1})
\end{cases}
\end{equation*}
\end{itemize}

where $\mathcal{B}(p)$ is Bernoulli random variables with probability $p$ and 
$\mathcal{N}(\mu,\sigma^2)$ is normal r.v. with mean $\mu$ and variance $\sigma^2$.
To generate ground truth $h(S, X_1, X_2)$, we use a linear model for LIN and IMF and non-linear methods for NLM SCM. In all of the synthetic models considered, we regard S as a binary sensitive attribute.
%
%
%
%
\subsection{Semi-Synthetic Data Model}

This semi-synthetic dataset encompasses gender, age, education, loan amount, duration, income, and saving variables, governed by the following structural equations:

\begin{align*}
\begin{cases}
G := U_G, \\
A := -35 + U_A, \\
E := -0.5 + \bigg(1 + e^{-\big(-1 + 0.5 G + (1 + e^{- 0.1 A})^{-1} + U_E \big)}\bigg)^{-1}, \\
L := 1 + 0.01 (A - 5) (5 - A) + G + U_L, \\
D := -1 + 0.1A + 2G + L + U_D, \\
I := -4 + 0.1(A + 35) + 2G + G E + U_I, \\
S :=  -4 + 1.5 \mathbb{I}_{\{I > 0\}} I + U_S, \\
Y \sim \text{Bernoulli}((1+e^{-0.3(-L-D+I+S+IS)})^{-1})
\end{cases}
\end{align*}

Where $\mathcal{B}$ and $\mathcal{G}$ represent the Bernoulli and Gamma distributions, respectively. The noise model were generated using the following formula:
\begin{align*}
\begin{cases}
U_G \sim \mathcal{B}(0.5) \\
U_A \sim \mathcal{G}(10, 3.5) \\
U_E \sim \mathcal{N}(0, 0.25) \\
U_L \sim \mathcal{N}(0, 4) \\
U_D \sim \mathcal{N}(0, 9) \\
U_I \sim \mathcal{N}(0, 4) \\
U_S\sim\mathcal{N}(0, 25) \\
\end{cases}
\end{align*}
We consider $G$ as a sensitive attribute.
%
%
%
%
\subsection{Real-World Data}
In our research, we have utilized the Adult dataset \cite{kohavi1996uci} and the COMPAS dataset \cite{washington2018argue} for our experimental analysis. To employ these datasets, we initially construct a SCM based on the causal graph proposed by \citet{nabi2018fair}.
For the Adult dataset, we incorporate features such as \textbf{sex}, \textbf{age}, \textbf{native-country}, \textbf{marital-status}, \textbf{education-num}, \textbf{hours-per-week}, and consider gender as a sensitive attribute.
In the case of the COMPAS dataset, the utilized features comprise \textbf{age}, \textbf{race}, \textbf{sex}, and \textbf{priors count}, which function as variables. Additionally, sex is considered a sensitive attribute.

For classification purposes, we apply data standardization prior to the learning process.

\begin{table*}
\centering
\begin{tabular}{ l  c@{\hspace{5pt}} c@{\hspace{5pt}} c@{\hspace{5pt}} | c@{\hspace{5pt}} c@{\hspace{5pt}} c@{\hspace{5pt}} | c@{\hspace{5pt}} c@{\hspace{5pt}} c@{\hspace{5pt}} | c@{\hspace{5pt}} c@{\hspace{5pt}} c@{\hspace{5pt}} | c@{\hspace{5pt}} c@{\hspace{5pt}} c@{\hspace{5pt}} | c@{\hspace{5pt}} c@{\hspace{5pt}} c@{\hspace{5pt}}}
	\toprule
    & \multicolumn{6}{c}{\textbf{Real-World Data}} 
	& \multicolumn{12}{c}{\textbf{Synthetic Data}}
    \\
	\cmidrule(lr){2-7}  \cmidrule(lr){8-19} 
	& \multicolumn{3}{c}{\textbf{Adult}} 
	& \multicolumn{3}{c}{\textbf{COMPAS}}
	& \multicolumn{3}{c}{\textbf{IMF}}
	& \multicolumn{3}{c}{\textbf{LIN}}
    & \multicolumn{3}{c}{\textbf{Loan}}
    & \multicolumn{3}{c}{\textbf{NLM}}
	\\
	\cmidrule(lr){2-4} \cmidrule(lr){5-7} \cmidrule(lr){8-10} \cmidrule(lr){11-13} \cmidrule(lr){14-16} \cmidrule(lr){17-19} 
	\multirow{1}{*}{\textbf{Trainer}} 
	& $M$
	& $U_{.01}$
    & $\cR_{.01}$
    & $M$
	& $U_{.01}$
    & $\cR_{.01}$
    & $M$
	& $U_{.01}$
    & $\cR_{.01}$
    & $M$
	& $U_{.01}$
    & $\cR_{.01}$
    & $M$
	& $U_{.01}$
    & $\cR_{.01}$
    & $M$
	& $U_{.01}$
    & $\cR_{.01}$
	\\
	\midrule
 AL & \textbf{0.53} & 0.19 & \textbf{0.01} & \textbf{0.35} & 0.15 & 0.01 & 0.27 & 0.28 & 0.1 & 0.19 & 0.9 & 0.24 & 0.63 & 0.27 & 0.16 & 0.34 & 0.54 & 0.35\\
CAL & 0.52 & 0.19 & 0.01 & 0.34 & 0.11 & \textbf{0.01} & 0.2 & 0.34 & 0.11 & 0.19 & 0.9 & 0.23 & 0.37 & 0.26 & 0.18 & 0.17 & 0.47 & 0.23\\
\textbf{CAPIFY} & 0.49 & \textbf{0.04} & 0.01 & 0.25 & \textbf{0.02} & 0.01 & \textbf{0.45} & \textbf{0.01} & \textbf{0.01} & 0.23 & \textbf{0.16} & \textbf{0.08} & 0.39 & \textbf{0.22} & 0.13 & 0.25 & \textbf{0.36} & 0.23\\
ERM & 0.53 & 0.18 & 0.01 & 0.35 & 0.18 & 0.01 & 0.44 & 0.03 & 0.01 & 0.34 & 0.43 & 0.16 & \textbf{0.7} & 0.35 & 0.22 & 0.42 & 0.55 & 0.38\\
LLR & 0.47 & 0.2 & 0.01 & 0.23 & 0.33 & 0.01 & 0.39 & 0.2 & 0.07 & 0.25 & 0.31 & 0.26 & 0.39 & 0.26 & \textbf{0.11} & 0.24 & 0.39 & \textbf{0.18}\\
ROSS & 0.39 & 0.11 & 0.03 & 0.26 & 0.17 & 0.01 & 0.44 & 0.02 & 0.01 & \textbf{0.36} & 0.45 & 0.1 & 0.67 & 0.38 & 0.26 & \textbf{0.43} & 0.53 & 0.4\\
	\midrule
\end{tabular}
\caption{The table presents additional results of our computational experiment in terms of Matthews coefficient (M, higher values are better), CAPI fairness metrics ($U_{.01}$, lower values are better), and the non-robust percentage concerning adversarial perturbation with radii $0.01$ ($\mathcal{R}_{.01}$, lower values are better). The best-performing methods for each trainer, dataset, and metric are highlighted in bold.}
\label{tab:sim2}
\end{table*}

%
%
\subsection{Training Methods}
In our study, we employ various training objectives to train the decision-making classifiers, denoted as $h(x)$. These training objectives are as follows:
\begin{itemize}
\item \textbf{Empirical risk minimization (ERM)}: This involves minimizing the expected risk with respect to the classifier parameters $\psi$, represented by 
\begin{equation*}
  \min_{\psi} {\bE}_{(v,y) \sim \cP_{\mathcal{D}}}[\ell(h_{\psi}(v), y)]  
\end{equation*}
\item \textbf{Adversarial Learning (AL)}: Adversarial learning involves training a model to withstand or defend against adversarial perturbation.
\begin{equation*}
  \min_\psi {\bE}_{(v,y) \sim \cP_{\mathcal{D}}}[ \max_{\delta \in B_{\Delta}(v)} \ell(h_\psi(v+\delta),y)]
\end{equation*}

\item \textbf{Local Linear Regularizer (LLR)}: This objective minimizes a combination of the expected risk, a regularization term involving gradients, and a term related to adversarial perturbations, given by 
\begin{align*}
\min_{\psi} {\bE}_{(v,y) \sim \cP_{\mathcal{D}}}&[\ell(h_{\psi}(v), y) + 
\mu_2 \norm{\nabla_v h(v)} + \\
& \mu_1  \max_{\norm{\delta} \leq \epsilon}|h(v+\delta) - \inner{\delta, \nabla_v  h(v)} -  h(v)|]
\end{align*}
\item \textbf{ROSS}: This method, based on a \citet{ross2021learning} work, seeks to minimize the expected risk along with an adversarial perturbation term, expressed as 
\begin{equation*}
    \min_{\psi} {\bE}_{(v,y) \sim \cP_{\mathcal{D}}}[\ell(h_{\psi}(v), y)+ \mu \min_{\delta \in B_{\Delta}(v)}\ell(h(v + \delta), \mathbf{1})]
\end{equation*}
\item \textbf{Causal Adversarial Learning (CAL)}: Causal adversarial perturbation is a component of the causal version of adversarial learning, rooted in the framework of CAP
\begin{equation*}
\min_\psi {\bE}_{(v,y) \sim \cP_{\mathcal{D}}}[ \max_{w \in B^{\CAP}_{\Delta}(v)} \ell(h_\psi(w),y)]
\end{equation*}
\item \textbf{CAPIFY}:
Exhibits similarities to the LLR method when the CAP is utilized as a perturbation attack.
\begin{equation*}
\begin{aligned}    
\min_{\psi} {\bE}_{(v,y) \sim \cP_{\mathcal{D}}}&[\ell(h_{\psi}(x), y) + \\
\quad & \mu_1 * \max_{s \in \mathcal{S}} \ell(h(\ddot{v}_s),y) \quad + \\
& \mu_2 *\gamma(\Delta, v) + \mu_3 *\| \nabla^{\mathcal{X}}_{v} f(v)\|_* ]
\end{aligned}
\end{equation*}
\end{itemize}
For our loss function $\ell$, we use the binary cross-entropy loss. 

%
%
%
%
\subsection{Hyperparameter Tuning}
The majority of the experimental setup is derived from the work of \citet{dominguez2022adversarial}. For each dataset associated with its corresponding label, we employ either the generalized linear model (GLM) or a multi-layer perceptron (MLP) with three hidden layers, each having a size of 100. The GLM is employed for LIN and IMF, while for other datasets, the MLP classifier is considered.

Each of these training objectives is applied to train six distinct datasets, utilizing 100 different random seeds. The optimization process employs the Adam optimizer with a learning rate of $10^{-3}$ and a batch size of 100. After optimizing the benchmark time and considering the training rate, we set the number of epochs to 10. To ensure comparability in benchmarking, we set all regularizer coefficients equal to 1.
%
%
%
%
\subsection{Metrics}

To compare the performance of trainers from accuracy, CAPI fairness, counterfactual fairness and adversarial robustness we use different 7 metrics as described below:

\begin{itemize}
    \item $\mathbf{A}$: The accuracy of a classifier is typically expressed as a percentage value.
    \item $\mathbf{M}$: The Matthews Correlation Coefficient (MCC) is a measure used in machine learning to assess the quality of binary classification models. The formula for the Matthews Correlation Coefficient is:
    $$\dfrac{(TP \times TN - FP \times FN)}{\sqrt{(TP + FP)(TP + FN) × (TN + FP)(TN + FN)}}$$
    Where:
    \begin{itemize}
        \item TP: True Positives
        \item TN: True Negatives
        \item FP: False Positives 
        \item FN: False Negatives
    \end{itemize}
    The MCC value ranges from $-1$ to $+1$, where $+1$ represents a perfect prediction, $0$ indicates random prediction, and $-1$ indicates perfect inverse prediction.
    \item $\mathbf{U_{\Delta}}$: Represents the proportion of data points located within the unfair area characterized by Def.~\ref{def:uai} with radius of $\Delta$. 
    \item $\mathbf{\mathcal{R}_{\Delta}}$: The notation $\mathcal{R}_{\Delta}$ denotes the fraction of data points that demonstrate non-robustness concerning adversarial perturbation with a radius of $\Delta$. This measure corresponds to the unfair area region in cases where a sensitive attribute is absent.
    \item $\mathbf{CF}$: represents the percentage of data points that exhibit unfairness concerning counterfactual fairness. This metric is equivalent to the unfair area when the perturbation radius is set to zero. 

\end{itemize}

%
%
%
%
\subsection{Additional Numerical Results}
The supplementary metrics are available in Table \ref{tab:sim2}. For the sake of simplicity, each value has been rounded to two decimal places, and the highest-performing indicator is highlighted. Figure \ref{fig:com_study_2} presents a bar plot accompanied by error bars that display the indicator values, along with the corresponding standard deviations attained through simulation.

The MCC score, much like accuracy, demonstrates a decrease when accounting for the incorporation of CAPIFY methods. However, in real datasets, the extent of this reduction is outweighed by the reduction in CAPI fairness. Notably, the performance of CAPIFY displays significant improvements on real datasets compared to synthetic datasets.

When considering the trade-off between accuracy and fairness, if we incorporate the combination of accuracy plus UAI as a measurement, the CAPIFY method emerges as the most effective algorithm. Notably, the CAPIFY method not only encapsulates counterfactual fairness but also demonstrates robustness. However, concerning robustness, certain methods that focus solely on robustness exhibit slightly superior performance compared to CAPIFY.

\begin{figure*}[ht]
\centering
\includegraphics[width=1\textwidth]{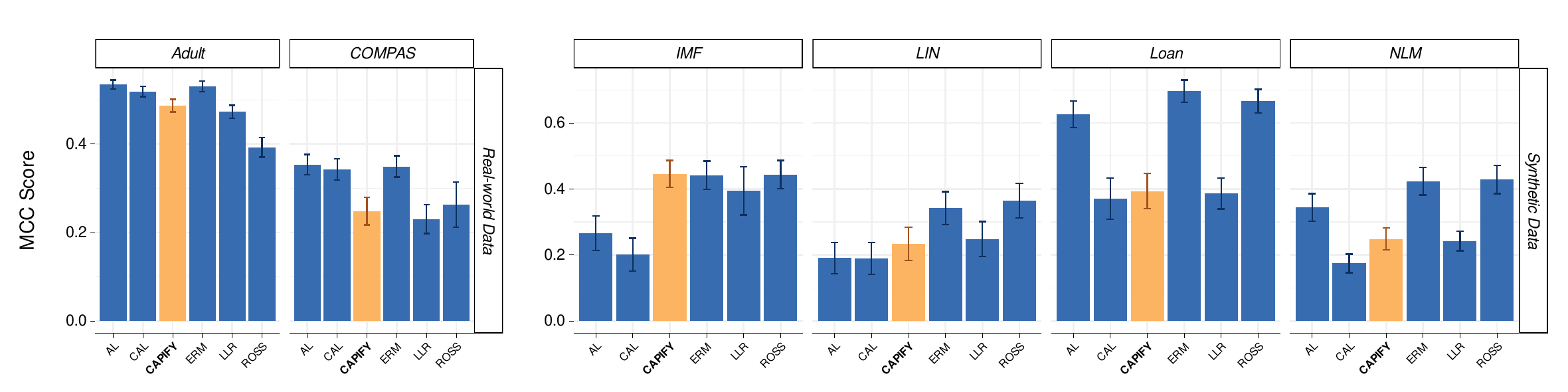}
\includegraphics[width=1\textwidth]{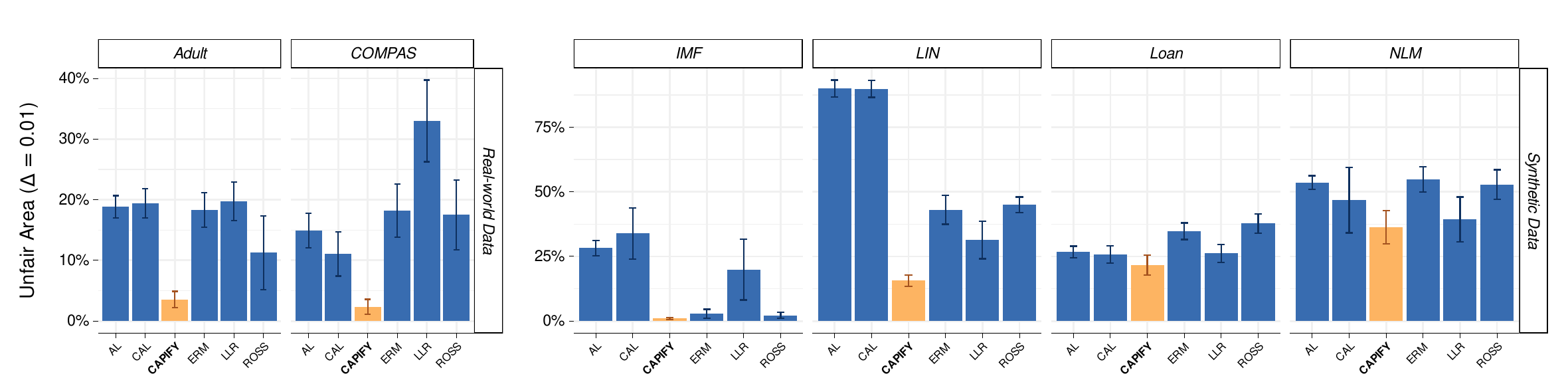}
\includegraphics[width=1\textwidth]{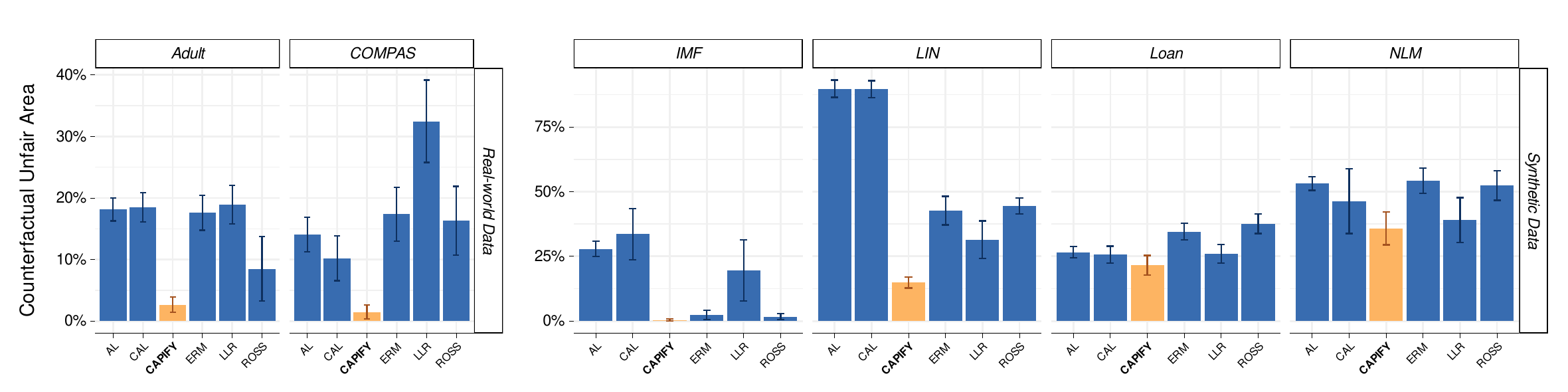}
\includegraphics[width=1\textwidth]{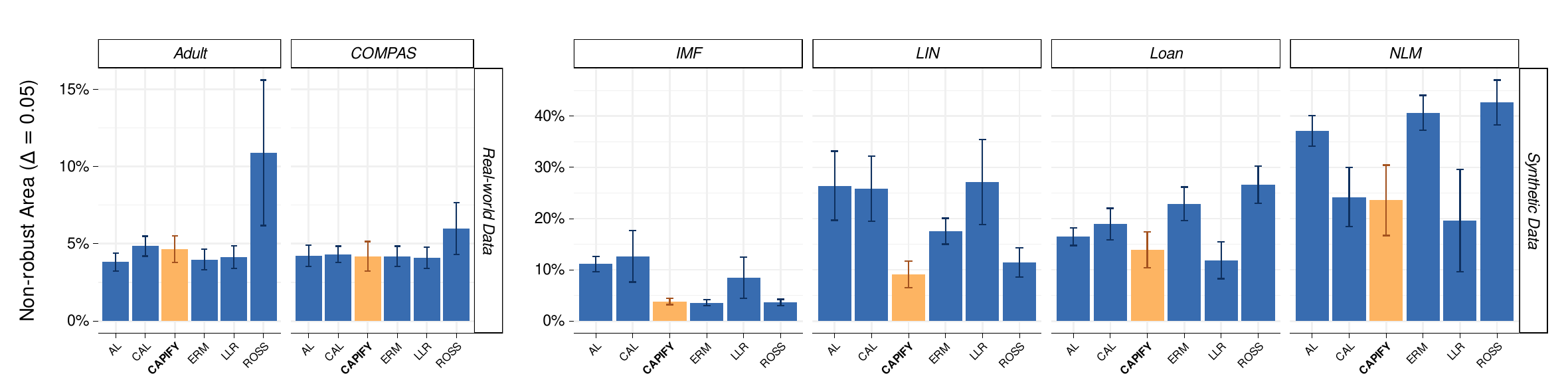}
\includegraphics[width=1\textwidth]{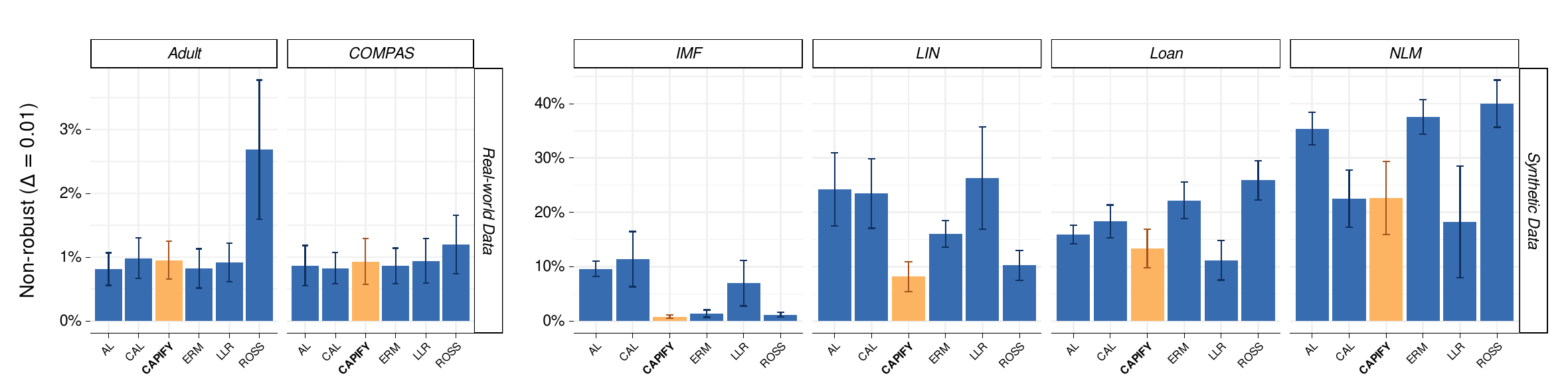}
\caption{The bar plots display the results of the computational experiments, including MCC scores, Unfair area ($\Delta = 0.01$), Unfair area for counterfactual fairness, and the percentage of non-robust area for $\Delta = 0.05$ and $\Delta = 0.01$.}
\label{fig:com_study_2}
\end{figure*}

\end{document}